\documentclass[twoside]{article}

\usepackage[accepted]{aistats2021}

\def\tdotoggle{0} %
\usepackage{amsfonts,amsthm,amssymb,mathtools,thmtools,thm-restate,bm,dsfont}
\usepackage[usenames,dvipsnames]{xcolor}
\usepackage{enumitem}
\usepackage{xspace}
\usepackage{array}
\usepackage{multirow,multicol}
\usepackage{subcaption} 
\usepackage[USenglish]{babel}
\usepackage{csquotes}
\usepackage{lpic}

\usepackage{hyperref}
\hypersetup{
	colorlinks=true,
	linkcolor=Sepia,
	citecolor=Sepia,
	filecolor=Sepia,
	urlcolor=Sepia
}

\usepackage[capitalize,noabbrev,nameinlink]{cleveref}
\newcommand{\secref}[1]{\hyperref[#1]{\S\ref{#1}}}
\newcommand{\Eqref}[1]{\hyperref[#1]{Eq. (\ref{#1})}}

\setitemize{itemsep=1pt,topsep=0pt,parsep=0pt,partopsep=0pt}

\usepackage[colorinlistoftodos,prependcaption,textsize=footnotesize]{todonotes} %
\providecommand{\tdotoggle}{1}
\ifnum\tdotoggle=1
\paperwidth=\dimexpr \paperwidth + 3.2cm\relax
\oddsidemargin=\dimexpr\oddsidemargin + 1.6cm\relax
\evensidemargin=\dimexpr\evensidemargin + 1.6cm\relax
\marginparwidth=\dimexpr \marginparwidth + 1.6cm\relax
\fi
\newcommand{\mytodo}[1]{\ifnum\tdotoggle=1{#1}\fi}

\newcommand{\info}[1]{\mytodo{\todo[linecolor=og,backgroundcolor=og!25,bordercolor=og]{#1}}}
\newcommand{\improve}[1]{\mytodo{\todo[linecolor=Gred,backgroundcolor=Gred!25,bordercolor=Gred]{#1}}}
\newcommand{\tableoftodos}{\ifnum\tdotoggle=1 \listoftodos[Comments/To Do's] \fi}

\definecolor{pw}{HTML}{7977B8}
\definecolor{og}{HTML}{3C8031}
\definecolor{maroon}{HTML}{AF3235}
\definecolor{yo}{HTML}{FAA21A}
\definecolor{mybrick}{RGB}{180,14,15}
\definecolor{Gred}{RGB}{219, 50, 54}
\definecolor{Ggreen}{RGB}{60, 186, 84}
\definecolor{Gblue}{RGB}{72, 133, 237}
\definecolor{Gyellow}{RGB}{247, 178, 16}
\definecolor{ToCgreen}{RGB}{0, 128, 0}
\definecolor{myGold}{RGB}{231,141,20}
\definecolor{myBlue}{rgb}{0.19,0.41,.65}
\definecolor{myPurple}{RGB}{175,0,124}

\definecolor{WolframOne}{rgb}{0.368417, 0.506779, 0.709798}
\definecolor{WolframTwo}{rgb}{0.880722, 0.611041, 0.142051}
\definecolor{WolframThree}{rgb}{0.560181, 0.691569, 0.194885}
\definecolor{WolframFour}{rgb}{0.922526, 0.385626, 0.209179}

\newcommand{\eps}{\varepsilon}
\newcommand{\indicator}{\mathds{1}}

\DeclareMathOperator{\Ex}{\mathbb{E}}

\DeclareMathOperator{\argmin}{\mathrm{argmin}}

\DeclareMathOperator{\Supp}{\mathrm{Supp}}
\newcommand{\what}[1]{\widehat{#1}}
\DeclareMathOperator{\sign}{sign}

\declaretheorem[name=Theorem]{theorem}

\declaretheorem[name=Fact,sibling=theorem]{fact}

\declaretheorem[name=Proposition,sibling=theorem]{proposition}

\declaretheorem[name=Setting]{setting}

\declaretheorem[name=Definition,sibling=theorem]{definition}

\newcommand{\set}[1]{\left \{ #1 \right \}}
\newcommand{\bit}{\{0,1\}}
\newcommand{\sbit}{\{-1,1\}}
\newcommand{\inabs}[1]{\left | #1 \right |}
\newcommand{\inparen}[1]{\left ( #1 \right )}
\newcommand{\insquare}[1]{\left [ #1 \right ]}

\newcommand{\inangle}[1]{\left \langle #1 \right \rangle}
\newcommand{\infork}[1]{\left \{ \begin{matrix} #1 \end{matrix} \right .}

\newcommand{\bbR}{\mathbb{R}}

\newcommand{\calD}{\mathcal{D}}
\newcommand{\calE}{\mathcal{E}}

\newcommand{\calI}{\mathcal{I}}

\newcommand{\calL}{\mathcal{L}}

\newcommand{\calS}{\mathcal{S}}

\newcommand{\calW}{\mathcal{W}}
\newcommand{\calX}{\mathcal{X}}
\newcommand{\calY}{\mathcal{Y}}
\newcommand{\calZ}{\mathcal{Z}}

\newcommand{\Etr}{\calE_{\mathrm{tr}}}
\newcommand{\Thetatr}{\Theta_{\mathrm{tr}}}

\newcommand{\logloss}{\ell_{\mathrm{log}}}
\newcommand{\sqloss}{\ell_{\mathrm{sq}}}

\newcommand{\Wlin}{\calW_{\mathrm{lin}}}
\newcommand{\Wscalar}{\calS} %

\newcommand{\Iscalar}{\calI_{\Wscalar}}

\newcommand{\IRM}{\mathsf{IRM}}
\newcommand{\IRMvone}{\mathsf{IRMv1}}
\newcommand{\ERM}{\mathsf{ERM}}

\newcommand{\Rad}{\mathrm{Rad}}

\newcommand{\ColoredMNIST}{\texttt{Colored-MNIST}\xspace}

\newcommand{\pred}{f} %

\usepackage[style=authoryear,maxbibnames=12,maxcitenames=2]{biblatex}
\defbibheading{bibliography}[\bibname]{%
  \subsubsection*{#1}%
  \markboth{#1}{#1}}
\renewbibmacro{in:}{}
\let\citep\parencite

\let\citet\textcite
\let\Citet\Textcite

\begin{document}

\ifnum\tdotoggle=1
\listoftodos
\setcounter{page}{0}
\fi

\runningauthor{Pritish Kamath, Akilesh Tangella, Danica J.\ Sutherland, Nathan Srebro}

\twocolumn[

\aistatstitle{Does Invariant Risk Minimization Capture Invariance?}

\aistatsauthor{%
    Pritish Kamath \\ {\small \tt pritish@ttic.edu} \And %
    Akilesh Tangella \\ {\small \tt akilesh@ttic.edu} \And %
    Danica J.\ Sutherland \\ {\small \tt dsuth@cs.ubc.ca}\And %
    Nathan Srebro \\ {\small \tt nati@ttic.edu} %
}

\aistatsaddress{Toyota Technological Institute at Chicago} ]

\begin{abstract}
We show that the Invariant Risk Minimization (IRM) formulation of Arjovsky et al. (2019) can fail to capture ``natural'' invariances, at least when used in its practical ``linear'' form, and even on very simple problems which directly follow the motivating examples for IRM. This can lead to worse generalization on new environments, even when compared to unconstrained ERM. The issue stems from a significant gap between the linear variant (as in their concrete method IRMv1) and the full non-linear IRM formulation.
Additionally, even when capturing the ``right'' invariances, we show that it is possible for IRM to learn a sub-optimal predictor, due to the loss function not being invariant across environments.
The issues arise even when measuring invariance on the population distributions, but are exacerbated by the fact that IRM is extremely fragile to sampling.
\end{abstract}

\section{INTRODUCTION}
Machine learning systems tend to seize on spurious correlations present in the training data,
and so when presented with out-of-distribution inputs,
they can fail spectacularly.
For instance, in the spirit of \citet{beery18cowscamels} and \citet{arjovsky19invariant}, consider a deep neural network trained to classify images %
as containing a cow or a camel.
Suppose that most pictures of cows in the training set are taken in (green) grassy pastures, and those of camels are mostly in (brown) deserts.
Then, the neural network is likely to strongly use background color for its predictions -- after all, it is a very easy signal to use, and it barely hurts the loss.
Such a network, however, will perform poorly at recognizing cows on a beach.
How, then, can we design a machine learning system to identify key features of interest -- face, shape, body color, etc., of animals -- and ignore spurious ones, like the background color?

Standard machine learning algorithms assume a training set independently sampled from a \emph{single} distribution, and seek good performance only on new samples from the same distribution.
There has been much work on models that can adapt to a new distribution given a small number of labeled samples \citep[see e.g. the survey of][]{domain-adaptation-theory},
or models that are robust to \emph{nearby} distributions \citep[see e.g. the survey of][]{dro-survey}.
Ideally though, we would hope for a model that can handle even \emph{large} changes in distribution, \emph{without} the need for labeled target samples.

In reality, our training data usually does \emph{not} actually come from a single homogeneous source: we may have collected it from different users, on different continents, in different years.
We thus may be able to tell which correlations are stable across environments (and hence are more likely to be the ``true'' correlations we seek),
and which behave differently in different environments (and are more likely to be spurious).

One approach, then, is to attempt to learn an \emph{invariant predictor}
\citep[e.g.][]{peters2015causal,heinzedeml18invariant,mateo18invariant}.
We might, for instance, assume that for the \emph{causally relevant} subset $S$ of the input variables $X$, the conditional distribution $\set{Y | X_S}$ is invariant across data sampled from different environments.
This usually requires assuming a meaningful causal graph relating the observed variables.
When classifying cows vs.\ camels based on image pixels,
such assumptions are not likely to hold on the input data,
though they could potentially apply to the latent variables underlying these images.

The \emph{Invariant Risk Minimization} ($\IRM$) framework of \citet{arjovsky19invariant}
tries to find a data representation $\varphi$  which discards the spurious correlations, leaving only the ``real'' signal,
by enforcing that the predictor $w$ acting on that representation
is simultaneously optimal in each environment given $\varphi$.
For instance, in the cows-vs-camels problem, $\varphi$ might remove the background color.
Since this gives a challenging bi-level optimization problem, \citeauthor{arjovsky19invariant} propose a relaxed version, $\IRMvone$,
which assumes $w$ is a linear predictor.
(We will overview the framework in \cref{sec:irm}.)
For a thorough overview of how this approach fits into the literature on out-of-domain generalization,
see the discussion by \citeauthor{arjovsky19invariant} and in particular Appendix A of \citet{gulrajani20lost}.
Subsequent work has provided new approaches for training in the $\IRM$ paradigm \citep[e.g.][]{ahuja20IRMgames,unshuffling}
and applications in domains such as interpretable language processing models \citep{chang20rationalization}.

Despite much initial promise, however, many key questions remain about the $\IRM$ framework:
how well does $\IRMvone$ approximate the exact version of the framework in general settings?
Do invariant predictors always generalize well on unseen environments?
When does a set of training environments allow us to find representations invariant across a broader set of target environments?
How does the framework and/or the algorithm behave on finite samples?

\paragraph{Our Contributions}
We advance the understanding of several core questions about the $\IRM$ framework.

In \cref{sec:bit-example}, we study a simple setting of environments over $\calX = \bit^2$, abstracting the \ColoredMNIST problem studied by \citet{arjovsky19invariant}.
We show that sometimes $\IRM$ with linear $w$ can provably fail to find a ``truly'' invariant predictor, even when solved with respect to the population loss, and even if we provide \emph{infinitely} many training environments.
In fact, it finds a predictor that is even {\em worse} on out-of-distribution environments than unrestricted $\ERM$.
This issue persists in the $\IRMvone$ implementation.

In \cref{sec:invariance-sufficient}, we note the population loss of even ``truly'' invariant predictors need not be invariant. We give a simple setting where $\IRM$, which minimizes loss over training environments, prefers an invariant predictor with worse out-of-distribution generalization.
In \cref{sec:invariance-generalize}, we study when it is possible to identify invariant predictors for a broad class of environments on the basis of a small range of training environments.
Although this is generally impossible, we show conditions on the environments under which it is possible.

Finally, in \cref{sec:sampling-issues}, we point out issues that arise when using the $\IRM$ paradigm over the distributions of empirical samples rather than the population distributions. Here, even invariant predictors (over the population distributions) might not be invariant when considered over the distribution of empirical samples.

\section{INVARIANT RISK MINIMIZATION}\label{sec:irm}
We now describe the $\IRM$ paradigm of \citet{arjovsky19invariant}.
We have a set of {\em environments} $\calE$,
where each environment $e \in \calE$ corresponds to a distribution $\calD_e$ over $\calX \times \calY$,
with $\calX$ being the space of inputs and $\calY$ that of outputs.
Our goal is to find a predictor $\pred : \calX \to \what{\calY}$;
we measure the quality of a prediction
with a loss function $\ell : \what{\calY} \times \calY \to \bbR_{\ge 0}$,
and the quality of a predictor by its \emph{population loss} on environment $e \in \calE$, given by
$\calL_e(\pred) := \Ex_{(x,y)\sim\calD_e} \ell(\pred(x),y)$.
In this paper, we mainly focus on the following special case.
\begin{setting}\label{setting:binary}
$\calY \subseteq \bbR$, $\what{\calY} = \bbR$,
and $\ell$ is either
the square loss $\sqloss(\what{y},y) := \frac{1}{2} (\what{y}-y)^2$,
or,
when $\calY = \sbit$ (corresponding to binary classification),
the logistic loss
$\logloss(\what{y},y) := \log(1+\exp(-\what{y}y))$.
\end{setting}

Given access to samples from some training environments $\Etr \subseteq \calE$, our aim to learn a predictor $\pred$ that minimizes the ``out-of-distribution'' loss over all environments in $\calE$, namely
\begin{equation}
\calL_{\calE}(\pred) := \sup_{e \in \calE} \calL_e(\pred)\tag{OOD-Gen}\label{eqn:ood-gen}
.\end{equation}

\subsection{Notions of Invariance}
The $\IRM$ paradigm attempts to solve this problem by learning 
an \emph{invariant} representation $\varphi : \calX \to \calZ$.
For instance, $\varphi$ might ``throw away'' the spurious background color in the cows-vs.-camels example, if $e_1 \in \calE$ is images from Ireland (where most cow images have grassy backgrounds), and $e_2 \in \calE$ is from India (with many more images of cows on city streets).
The formal definition of \emph{invariant} is as follows.

\begin{definition}[Definition 3 of \cite{arjovsky19invariant}]\label{def:all-invariant}
A representation\footnote{We always assume $\varphi$ and $w$ are measurable. For further subtleties with \cref{def:all-invariant,def:space-invariant}, see \cref{appendix:def-subtleties}.} $\varphi : \calX \to \calZ$ is {\em invariant} over a set of environments $\calE$ if there exists a $w : \calZ \to \what{\calY}$ such that $w$ is simultaneously optimal on $\varphi$ for all environments $e \in \calE$, that is, $w \in \argmin_{\overline{w} : \calZ \to \what{\calY}} \calL_e(\overline{w} \circ \varphi)$.%
\end{definition}
This definition is motivated by the following observation of  \citet{arjovsky19invariant},
which corresponds more closely to an intuitive definition of invariance.
\begin{restatable}{observation}{obsexpectations}\label{obs:expectations}
Under \cref{setting:binary},
a representation $\varphi : \calX \to \calZ$ is invariant over $\calE$ if and only if for all $e_1, e_2 \in \calE$, it holds that
\[
    \Ex_{\calD_{e_1}}[Y \mid \varphi(X) = z] = \Ex_{\calD_{e_2}}[Y \mid \varphi(X) = z]
\]
for all $z \in \calZ_\varphi^{e_1} \cap \calZ_\varphi^{e_2}$,
where $\calZ_\varphi^e$ are the representations from $\calD_e$, $\calZ_\varphi^e \coloneqq \{ \varphi(X) \mid (X, Y) \in \Supp(\calD_{e}) \} $.
\end{restatable}%
We give a proof in \cref{appendix:irm-proofs} for completeness.

Crucially, \cref{def:all-invariant} requires that $\varphi$ and $w$ are unrestricted in the space of {\em all} (measurable) functions.
However, we wish to learn $\varphi$ and $w$ with access to only (finite) training sets $S_e$ sampled from $\calD_e$,
for only a small subset of training environments $\Etr \subseteq \calE$.
For this to be feasible,
it is natural to add a restriction that $\varphi \in \Phi$ and $w \in \calW$,
for suitable classes $\Phi$ of functions mapping $\calX \to \calZ$ 
and $\calW$ of functions mapping $\calZ \to \what\calY$.
Any choice of function classes $(\Phi, \calW)$ defines a class of ``invariant'' predictors for a set of environments $\calE$.

\begin{definition}\label{def:space-invariant}
For any $\Phi$, $\calW$ and loss function $\ell$, the set of invariant predictors on $\calE$, $\calI_{\Phi,\calW}^{\ell}(\calE)$,
is the set of all predictors $\pred : \calX \to \what{\calY}$
such that $\exists\, (w,\varphi) \in \calW \times \Phi$ satisfying the following:
\begin{itemize}
\item $\pred = w \circ \varphi$, and
\item for all $e \in \calE$, $w \in \argmin_{\overline{w} \in \calW} \calL_e(\overline{w} \circ \varphi)$.
\end{itemize}
For ease of notation, we will keep the loss function $\ell$ implicit. When $\Phi$ is the space of all functions $\calX \to \calZ$, we denote $\calI_{\Phi,\calW}(\calE)$ as simply $\calI_{\calW}(\calE)$.
Moreover, when $\calW$ is the space of all functions $\calZ \to \what{\calY}$,
we denote $\calI_{\calW}(\calE)$ as $\calI(\calE)$,
leaving the choice of $\calZ$ implicit.\footnote{In defining $\calI(\calE)$, the choice of $\calZ$ does not matter, as long as $\calZ$ is large enough compared to $\calX$; for instance, $\calZ = \calX$ is always a valid choice.}
\end{definition}

Because exact optimization over $\calW$ is in general difficult,
it is useful to consider some special cases.
A natural option is
{\em linear} invariant predictors, where $\calZ = \bbR^d$ and $\calW = \Wlin^d$ is the space of all linear functions on $\bbR^d$.
\Citet{arjovsky19invariant} argued that
linear predictors in fact provide no additional representation advantage over
{\em scalar} invariant predictors, the linear predictors for $d = 1$,
$\calW = \Wscalar := \Wlin^1$. In our notation, this translates to the following lemma, proved in \cref{appendix:irm-proofs}.

\begin{restatable}{lemma}{lemirmfullvslinear}\label{lem:irm-full-vs-linear}
Under \cref{setting:binary}, for all $\calE$ and $d \ge 1$,
\[ \calI(\calE) ~\subseteq~ \Iscalar(\calE) ~=~ \calI_{\Wlin^d}(\calE) .\]
\end{restatable}

\subsection{Algorithms}
Armed with a notion of invariance,
we still need a way to pick an invariant predictor based on training environments $\Etr \subseteq \calE$.
\citet{arjovsky19invariant} proposed the {\em Invariant Risk Minimization} objective given by
\begin{gather*}
    \min_{\substack{\varphi : \calX \to \calZ\\w : \calZ \to \what{\calY}}}\ \  \sum_{e \in \Etr} \calL_e(w \circ \varphi)
    \\
    \mathrm{s.t.} \, \forall e \in \Etr, \;
    w \in \argmin_{\overline{w} : \calZ \to \what\calY} \calL_e(\overline{w} \circ \varphi)
,\end{gather*}
which in our notation is equivalent to
\begin{gather}\label{eqn:irm}
\tag{$\IRM$}
\min_{\pred \in \calI(\Etr)}\ \ \sum_{e \in \Etr} \calL_e(\pred).\end{gather}
We can analogously define
$\IRM_{\calW}$ to choose a predictor $\pred \in \calI_{\calW}(\Etr)$,
and $\IRM_{\Phi,\calW}$ from $\pred \in \calI_{\Phi,\calW}(\Etr)$.%

Characterizing $\calI_{\calW}(\Etr)$ is difficult in general; fortunately $\calI_{\Wlin^d}(\calE) = \Iscalar(\calE)$ affords a simple characterization.
Any predictor $\pred \in \Iscalar(\Etr)$
can be written as $\pred(x) = w_* \, \varphi_*(x)$ for a scalar $w_*$.
Without loss of generality,
we can simply absorb the scalar $w_*$ into $\varphi \coloneqq w_* \, \varphi_*$,
so that $\pred = 1 \cdot \varphi$.
In \cref{setting:binary}, where the loss function is convex and differentiable,
$\pred = 1 \cdot \varphi \in \Iscalar(\Etr) = \calI_{\Wlin^d}(\Etr)$ if and only if
\begin{equation} \label{eq:grad-0-scalar}
\tag{$\nabla_w$}
\text{for all } e \in \Etr , \quad \nabla_{w|w=1} \calL_e(w \cdot \varphi) = 0
.\end{equation}
Yet, $\IRM_{\Wscalar}$ remains a bi-level optimization problem. For practical purposes,
\citet{arjovsky19invariant} proposed to soften this hard constraint,
giving the algorithm $\IRMvone$ to approximate $\IRM_{\Wscalar}$:
\begin{equation}\label{eqn:irmv1}
\tag{$\IRMvone$}\min_{\varphi : \calX \to \bbR} \ \ \sum_{e \in \Etr} \calL_e(\varphi) + \lambda  \inabs{\nabla_{w|w=1} \calL_e(w \cdot \varphi)}^2
.\end{equation}
A natural baseline is the $\ERM$ algorithm,
which simply minimizes the loss over training environments:
\begin{gather}\label{eqn:erm}
\tag{$\ERM$}\min_{\pred : \calX \to \what{\calY}} \ \ \sum_{e \in \Etr} \calL_e(\pred).
\end{gather}
While we referred to $\IRM$, $\IRMvone$ and $\ERM$ as ``algorithms'' above, there still remain two key details that make these impractical as stated:
(i) the loss minimized refers to the {\em population loss}, to which we do not have direct access,
and (ii) we are assuming that $\varphi$ is unrestricted in the space of all functions.
\Citet{arjovsky19invariant} attempt to remedy these issues in $\IRMvone$ by
(i) replacing the population loss by the corresponding empirical loss measured over training sets, %
and (ii) by optimizing $\varphi$ over a sufficiently expressive parameterized model, such as a deep neural network, using gradient-based local search methods.

Nevertheless, as we discuss shortly, $\IRM_{\Wscalar}$ does not capture $\IRM$ even when operating on the population loss with unrestricted $\varphi$.
Unless otherwise stated, we always consider $\IRM$, $\IRM_{\Wscalar}$, $\IRMvone$ and $\ERM$ as operating over population losses.

\subsection{Related Work}
\Citet{rosenfeld20risksofirm} demonstrate an example where there exists a {\em near-optimal} solution to the $\IRMvone$ objective, that nearly matches performance of $\IRM$ on training environments, but does no better than $\ERM$ on environments that are ``far'' away from the training distributions. This example relies on environments which barely overlap, allowing the representation to simply ``memorize'' the training environments.
Indeed, \citet{ahuja20sample} argue that $\IRM$ can have an advantage over $\ERM$ only when the support of the different environment distributions have a significant overlap.
\Citet{gulrajani20lost} find empirically that with current models and data augmentation techniques, $\ERM$ achieves state-of-the-art practical performance in domain generalization.\info{see comments in \LaTeX\xspace source for some papers that were not cited.}
\Citet{nagarajan2021understanding}, meanwhile, theoretically study the behavior of $\ERM$ for domain generalization.

Note that in prior work, $\IRM_{\Wscalar}$/$\IRMvone$ and $\IRM$ are often referred to interchangeably. As we demonstrate, $\IRM_{\Wscalar}$ can behave very differently from $\IRM$, even on simple examples that motivated the $\IRM$ approach.

\section{COLORED-MNIST AND TWO-BIT ENVIRONMENTS}\label{sec:bit-example}
To illustrate the utility of the $\IRM$ approach and $\IRMvone$ in particular, \citet{arjovsky19invariant} introduced the \ColoredMNIST problem, a synthetic task derived from \texttt{MNIST} \parencite{mnist}. While \texttt{MNIST} images are grayscale, in \ColoredMNIST each image is colored either red or green in a way that correlates strongly (but spuriously) with the class label. Here $\ERM$ learns to exploit the color, and fails at test time when the direction of correlation with the color is reversed.

To understand the behavior of $\IRM_{\Wscalar}$ and $\IRMvone$ on \ColoredMNIST,
we study an abstract version based on two bits of input,
where $Y$ is the binary label to be predicted,
$X_1$ corresponds to the label of the handwritten digit (0-4 or 5-9),
and $X_2$ corresponds to the color (red or green).
We represent each environment $e$ with two parameters $\alpha_e, \beta_e \in [0,1]$.
The distribution $\calD_e$ is defined as
\begin{align}
    Y &~\gets~ \Rad(0.5),\nonumber\\
    X_1 &~\gets~ Y \cdot \Rad(\alpha_e),\label{eqn:two-bit-envs}\tag{Two-Bit-Envs}\\
    X_2 &~\gets~ Y \cdot \Rad(\beta_e)\nonumber %
,\end{align}
where $\Rad(\delta)$ is a random variable taking value $-1$ with probability $\delta$ and $+1$ with probability $1-\delta$.
For convenience, we denote an environment $e$ as $(\alpha_e, \beta_e)$.

Following the experiments with \ColoredMNIST as done by \citet{arjovsky19invariant}, we consider a set of environments $\calE_{\alpha} \coloneqq \set{(\alpha, \beta_e) : 0 < \beta_e < 1}$.
It can be shown that there only two predictors in $\calI(\calE_{\alpha})$, one being the trivial $0$-predictor, and another that depends only on $X_1$ (see proof of \cref{prop:identify-invariances} for details).

\paragraph{Motivating example of \citet{arjovsky19invariant}} Consider $\calE = \calE_{0.25}$ and $\Etr = \set{(0.25,0.1), (0.25, 0.2)}$. Focusing on the case of $\sqloss$, \eqref{eqn:erm} on $\Etr$ learns the predictor $\pred_{\ERM}$ that is (approximately) given by
\newcommand{\predictor}[6][]{
	\begin{center}{\renewcommand{\arraystretch}{1.2}
	\begin{tabular}{|l!{\vrule width 1.3pt}c|c|}
	\hline
	#2 & $X_2=1$ & $X_2=-1$\\
	\noalign{\hrule height 1.3pt}
	$X_1 = 1$ & #3 & #4\\
	\hline
	$X_1 = -1$ & #5 & #6\\
	\hline
	\end{tabular}} #1
	\end{center}
}
\predictor[;]{$\pred_{\ERM}$}{$0.8889$}{$-0.3077$}{$0.3077$}{$-0.8889$}
the prediction clearly depends on $X_2$ as well as $X_1$.
On each environment in $\Etr$, the signal from $X_2$ is stronger than that from $X_1$,
and so the binary predictor here can be summarized as
$\sign(\pred_\ERM(X)) = \sign(X_2)$.
On the other hand, \eqref{eqn:irm} chooses the predictor $\pred_{\IRM}$ %
\predictor[,]{$\pred_{\IRM}$}{$0.5$}{$0.5$}{$-0.5$}{$-0.5$}
whose binary behavior is $\sign(\pred_\IRM(X)) = \sign(X_1)$.

On $e \in \Etr$,
$\pred_\ERM$ achieves a lower loss than $\pred_\IRM$,
since it is using the more powerful signal $X_2$.
But, if we evaluate the ability of these predictors to generalize far out of distribution
to a case where the (spurious) correlation of $X_2$ has flipped entirely,
$e = (0.25, 0.9)$,
$\pred_\ERM$ will give the wrong (binary) prediction 90\% of the time,
and get square loss $\calL_e(\pred_\ERM) = 0.985$.
This is far worse than $\pred_\IRM$,
which at $\calL_e(\pred_\IRM) = 0.375$ has not suffered at all compared to $\Etr$. It is even worse than the trivial $0$-predictor, $\calL_e(\pred_0) = 0.5$.

It turns out that $\IRM_{\Wscalar}$ also learns the predictor $\pred_{\IRM}$ here, demonstrating the utility of this relaxation of $\IRM$. This raises a natural question:

{\em Does $\IRM_{\Wscalar}$ always learn the same predictor as $\IRM$?}

\citet[Section 4.1]{arjovsky19invariant} considered a specialized {\em linear} family of environments, where they proved that indeed $\IRM_{\Wscalar}$ learns an invariant predictor, as learned by $\IRM$, for any $\Etr$ with a sufficient number of environments in ``general position.''\footnote{The problem \eqref{eqn:two-bit-envs} does not fit the setting of their Theorem 9, because flipping signs cannot be phrased as independent additive noise.} \parencite[See also][Section 5.]{rosenfeld20risksofirm} It was left to future work whether $\IRM_{\Wscalar}$ learns invariant predictors in the sense of $\IRM$ more generally as well.

\paragraph{\boldmath A failure mode of $\IRM_{\Wscalar}$ and $\IRMvone$} We show that in fact for a simple set of two-bit environments, $\IRM_{\Wscalar}$ finds a predictor worse than that learned by $\IRM$, and even worse than the one learned by $\ERM$. %

This occurs, e.g., for $\calE = \calE_{0.1}$
with training environments $\Etr = \set{e_1 = (0.1,0.2), e_2 = (0.1,0.25)}$. 
The learned predictors are (approximately) as follows.
\predictor{$\pred_{\ERM}$}{$\phantom{-}0.9375$}{$\phantom{-}0.4464$}{$-0.4464$}{$-0.9375$}
\predictor{$\pred_{\IRM}$}{$\phantom{-}0.8$}{$\phantom{-}0.8$}{$-0.8$}{$-0.8$}
\predictor{$\pred_{\IRM_{\Wscalar}}$}{$\phantom{-}0.9557$}{$\phantom{-}0.2943$}{$-0.2943$}{$-0.9557$}
$X_1$ is the stronger signal for $Y$ in this $\Etr$,
and all of these predictors make the same binary predictions,
but with differing amounts of confidence.
Extrapolating to the same kind of test environment where the correlation of $X_2$ has flipped,
$e_\mathrm{test} = (0.1, 0.9)$,
we observe the following (approximate) losses:
\begin{center}{\renewcommand{\arraystretch}{1.2}
\begin{tabular}{|l!{\vrule width 1.3pt}cccc|}
\hline
& $\pred_{\ERM}$ & $\pred_{\IRM}$ & $\pred_{\IRM_{\Wscalar}}$ & $\pred_0$\\
\hline
$\calL_{e_{1}}(\cdot)$ & $0.15$ & $0.18$ & $0.15$ & $0.5$\\
$\calL_{e_{2}}(\cdot)$ & $0.16$ & $0.18$ & $0.17$ & $0.5$\\
$\calL_{e_{\mathrm{test}}}(\cdot)$ & $0.28$ & $0.18$ & $0.38$ & $0.5$\\
\hline
\end{tabular}} .
\end{center}
The relation between $\IRM$ and $\ERM$ is as expected: $\IRM$ trades slightly worse loss on the training environments for much better extrapolation to the distant environment $e_{\mathrm{test}} = (0.1, 0.9) \in \calE_{0.1}$.
But while $\IRM_\Wscalar$ also suffers slightly on the training environments,
it is even worse than $\ERM$ at extrapolation to $e_{\mathrm{test}}$!
The invariant feature $X_1$ is more correlated with $Y$ than the non-invariant feature $X_2$ in all of the training environments, and yet $\IRM_{\Wscalar}$ depends on $X_2$ even \emph{more} seriously than $\ERM$ does.

Moreover, this is not a carefully-selected pathological example that would go away with more training environments. In fact, $\IRM_{\Wscalar}$ chooses the same predictor even if we include {\em any number of} additional training environments $(0.1,\beta_e)$ for $\beta_e < 0.28$. Indeed, we show that for these two-bit environments $\calE_{\alpha}$, any two training environments are sufficient to recover the set of all invariant predictors (proof in \cref{appendix:two-bits}).

\begin{restatable}{proposition}{propidentifyinvariances}\label{prop:identify-invariances}
Under \cref{setting:binary}, for all $\alpha \in (0,1)$ and $\Etr = \set{e_1, e_2}$ for any two distinct $e_1, e_2 \in \calE_{\alpha}$,
\[ \text{(i) } \Iscalar(\Etr) = \Iscalar(\calE_{\alpha})
\quad \text{ and } \quad
\text{(ii) } \calI(\Etr) = \calI(\calE_{\alpha}).\]
\end{restatable}

Thus, the issue is not just that we have don't have enough training environments.
Rather, as we will now show, what $\IRM_\Wscalar$ determines to be an ``invariant predictor'' is broader than our intuitive sense -- or $\IRM$'s notion -- of what it means to be invariant.

\paragraph{\boldmath Predictors in $\Iscalar(\calE_{\alpha})$}
Recall a predictor $\pred = 1 \cdot \varphi$ is in $\Iscalar(\Etr)$ if and only if $\varphi$ satisfies \cref{eq:grad-0-scalar}.

For $\sqloss$, this is same as having that for all $e \in \Etr$,
\[ \left . \frac{\partial}{\partial w} \inparen{\Ex\limits_{(X,Y)\sim\calD_e} ~ \frac{(w \cdot \varphi(X) - Y)^2}{2}} \right \rvert_{w=1} ~=~ 0\, , \]
or equivalently,
\begin{equation}\label{eqn:grad-constraint}
\Ex\limits_{(X,Y)\sim\calD_e} ~ (\varphi(X) - Y)\cdot\varphi(X) ~=~ 0
\tag{$\nabla_w$ for $\sqloss$}
.\end{equation}
This is a system of quadratic polynomials in four variables $\set{\varphi(x) : x \in \sbit^2}$.
For ease of visualization, we focus on {\em odd} predictors $\pred = 1\cdot \varphi \in \Iscalar(\Etr)$, namely those satisfying $\pred(x) = - \pred(-x)$ for all $x \in \sbit^2$.
This choice is motivated by the symmetry present in $\calD_e$ and the loss $\sqloss$, along with the observation that the predictors $\pred_{\ERM}$, $\pred_{\IRM}$ and $\pred_{\IRM_{\Wscalar}}$ are all odd.
This allows us to focus on just two variables $\varphi(1,1) = -\varphi(-1,-1)$ and $\varphi(1,-1) = -\varphi(-1,1)$.

\Cref{fig:sqls_grad_constraints} shows the solutions of \eqref{eqn:grad-constraint} among all odd $\varphi$
for four environments in $\calE_{0.1}$. %
There are precisely four odd choices of $\varphi \in \Iscalar(\calE_{0.1}) = \Iscalar(\Etr)$.
Two are the expected solutions $\pred_0$ and $\pred_{\IRM}$ described above; these are the only two predictors in $\calI(\Etr) = \calI(\calE_{0.1})$.
$\Iscalar(\calE_{0.1})$, however, contains two \emph{more} odd predictors, $\pred_1$ and $\pred_2$,
the former being $\pred_{\IRM_\Wscalar}$ from above.
$\pred_{\IRM_\Wscalar}$ achieves a smaller loss than the other solutions for the two training environments $(0.1, 0.2)$ and $(0.1, 0.25)$,
but higher loss than $\pred_\IRM$ for environments $(0.1, 0.4)$ or $(0.1, 0.9)$.
\Cref{fig:scalar-invariant-losses} visualizes the losses of these four odd predictors on environments with varying $\beta_e$.
\Cref{apx:two-bits-sqloss} has more details, including an analysis that explains precisely when these counterexamples arise.\footnote{This analysis was communicated to us by L\'eon Bottou.}

\begin{figure}[t]
\centering
\begin{lpic}[l(3mm),b(3mm)]{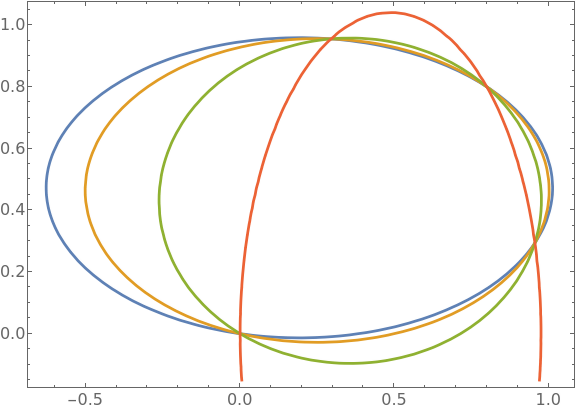(.5)}
\scriptsize
\lbl[c]{-6,60,90;\textcolor{black!60}{$\varphi(1,1)=-\varphi(-1,-1)$}}
\lbl[c]{80,-5;\textcolor{black!60}{$\varphi(1,-1)=-\varphi(-1,1)$}}
\tiny
\lbl[c]{35,90,19;\textcolor{WolframOne}{\boldmath$e=(0.1,0.2)$}}
\lbl[c]{39,76,40;\textcolor{WolframTwo}{\boldmath$e=(0.1,0.25)$}}
\lbl[c]{54,73,50;\textcolor{WolframThree}{\boldmath$e=(0.1,0.4)$}}
\lbl[c]{73,65,71;\textcolor{WolframFour}{\boldmath$e=(0.1,0.9)$}}
\normalsize
\lbl[c]{68,27,0;\textcolor{black!70}{\boldmath$\pred_0$}}
\lbl[c]{133,88,0;\textcolor{black!70}{\boldmath$\pred_{\IRM}$}}
\lbl[c]{88,88,0;\textcolor{black!70}{\boldmath$\pred_1$}}
\lbl[c]{127,45,0;\textcolor{black!70}{\boldmath$\pred_2$}}
\end{lpic}
\caption{Odd solutions to \eqref{eqn:grad-constraint} for four environments in $\calE_{0.1}$.%
}
\label{fig:sqls_grad_constraints}
\end{figure}

\begin{figure}
\centering
\begin{lpic}[t(4mm),r(5mm)]{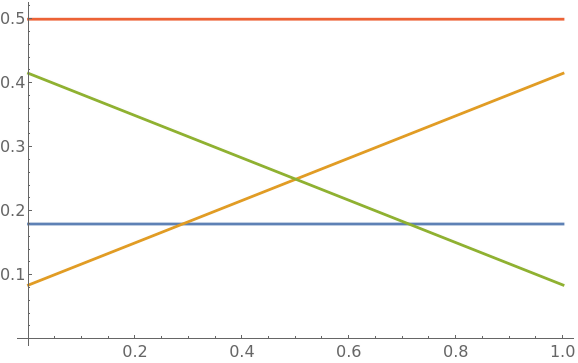(0.5)}
\normalsize
\lbl[c]{10,97;\textcolor{black!60}{\boldmath $\calL_e(\pred)$}}
\lbl[c]{153,6;\textcolor{black!60}{\boldmath $\beta_e$}}
\lbl[c]{75,30;\textcolor{WolframOne}{\boldmath$\pred_{\IRM}$}}
\lbl[c]{24,20;\textcolor{WolframTwo}{\boldmath$\pred_1$}}
\lbl[c]{130,20;\textcolor{WolframThree}{\boldmath$\pred_2$}}
\lbl[c]{75,82;\textcolor{WolframFour}{\boldmath$\pred_0$}}
\end{lpic}
\caption{Losses $\calL_e$ (for $\ell = \sqloss$) of odd predictors in $\Iscalar(\calE_{0.1})$ for various $e = (0.1,\beta_e)$.}
\label{fig:scalar-invariant-losses}
\end{figure}

Thus, $\IRM_{\Wscalar}$ can find representations $\varphi$ which are not {\em invariant} in the sense of \cref{def:all-invariant}.
In particular, for $\calE_{0.1}$ with $\sqloss$,
$\IRM_\Wscalar$'s feasible set of solutions is
$\Iscalar(\Etr) \supsetneq \calI(\Etr)$, or equivalently $\Iscalar(\calE_{0.1}) \supsetneq \calI(\calE_{0.1})$.

As seen from \cref{fig:scalar-invariant-losses}, $\pred_{\IRM_{\Wscalar}} = \pred_1$ has the lowest loss of those four solutions for $\beta_e \le 0.28$.
More training environments will not help $\IRM_{\Wscalar}$ pick $\pred_{\IRM}$,
unless the average value of $\beta_e$ across environments $e \in \Etr$ is between $0.29$ and $0.71$. If the average value of $\beta_e$ exceeds $0.72$, $\IRM_{\Wscalar}$ switches to the other solution $\pred_2$.

We know that $\IRMvone$ becomes exactly $\ERM$ when its regularization weight is $\lambda = 0$,
and $\IRM_\Wscalar$ for $\lambda = \infty$.
\Cref{fig:irmv1_interpolation}
shows\footnote{The $\IRMvone$ objective can be non-convex, even for $\sqloss$, and typical optimization algorithms sometimes find local minima. We instead solved $\IRMvone$ by explicitly enumerating the (odd) stationary points.}
the solution smoothly interpolating between
$\pred_{\ERM}$ and $\pred_{\IRM_{\Wscalar}}$, with the reliance on $X_2$ increasing as $\lambda \to \infty$.

\begin{figure}[t]
\centering
\begin{lpic}[r(10mm)]{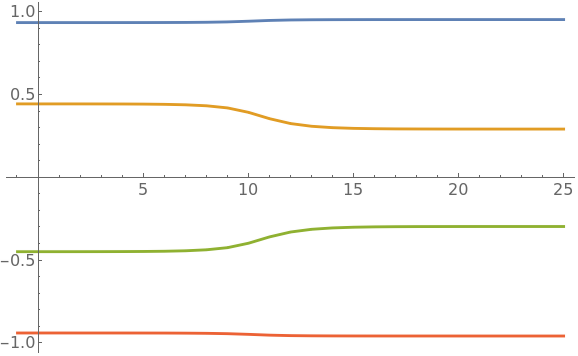(0.46)}
\scriptsize
\lbl[c]{158,45;\textcolor{black!60}{\boldmath $\log_2(\lambda)$}}
\tiny
\lbl[r]{142,81;\textcolor{WolframOne}{\boldmath$\varphi(1,1)$}}
\lbl[r]{142,62;\textcolor{WolframTwo}{\boldmath$\varphi(1,-1)$}}
\lbl[r]{142,29;\textcolor{WolframThree}{\boldmath$\varphi(-1,1)$}}
\lbl[r]{142,9;\textcolor{WolframFour}{\boldmath$\varphi(-1,-1)$}}
\end{lpic}
\caption{$\IRMvone$ on $\Etr = \set{(0.1,0.2),(0.1,0.25)}$. The horizontal axis is $\log_2(\lambda)$, with $-1$ representing $\lambda = 0$.}
\label{fig:irmv1_interpolation}
\end{figure}

\paragraph{\boldmath $\logloss$ loss} A similar failure mode occurs for $\logloss$ on $\calE_{0.05}$ when training on $\Etr = \{(0.05,0.1),(0.05,0.2)\}$. We give more details in \cref{apx:two-bits-logloss}.

\subsection{Experiments with \ColoredMNIST} \label{sec:colored-mnist}
We now confirm that the failure mode studied above can also arise in practical training of deep networks based on $\IRMvone$.
\ColoredMNIST corresponds to the two-bit environments above,
where $X_1$ is a (grayscale) image from \texttt{MNIST},
and $X_2$ is a color (red or green) which is assigned to that image.%
\footnote{In practice, we sample the image $X_1$ first and then flip $Y$ with probability $\alpha_e$; this is equivalent.}
Thus, a learning algorithm
which finds global minima of the $\IRMvone$ population-level objective
in a model capable of perfectly classifying \texttt{MNIST} digits
would behave exactly as described above.
In practice, however, we optimize empirical estimates of the risk and gradient penalty,
in a model class which may not contain an exactly perfect digit classifier,
with an algorithm which may not find the global optimum.

\begin{figure}[t]
    \centering
    \begin{subfigure}[b]{\columnwidth}
        \includegraphics[width=\textwidth]{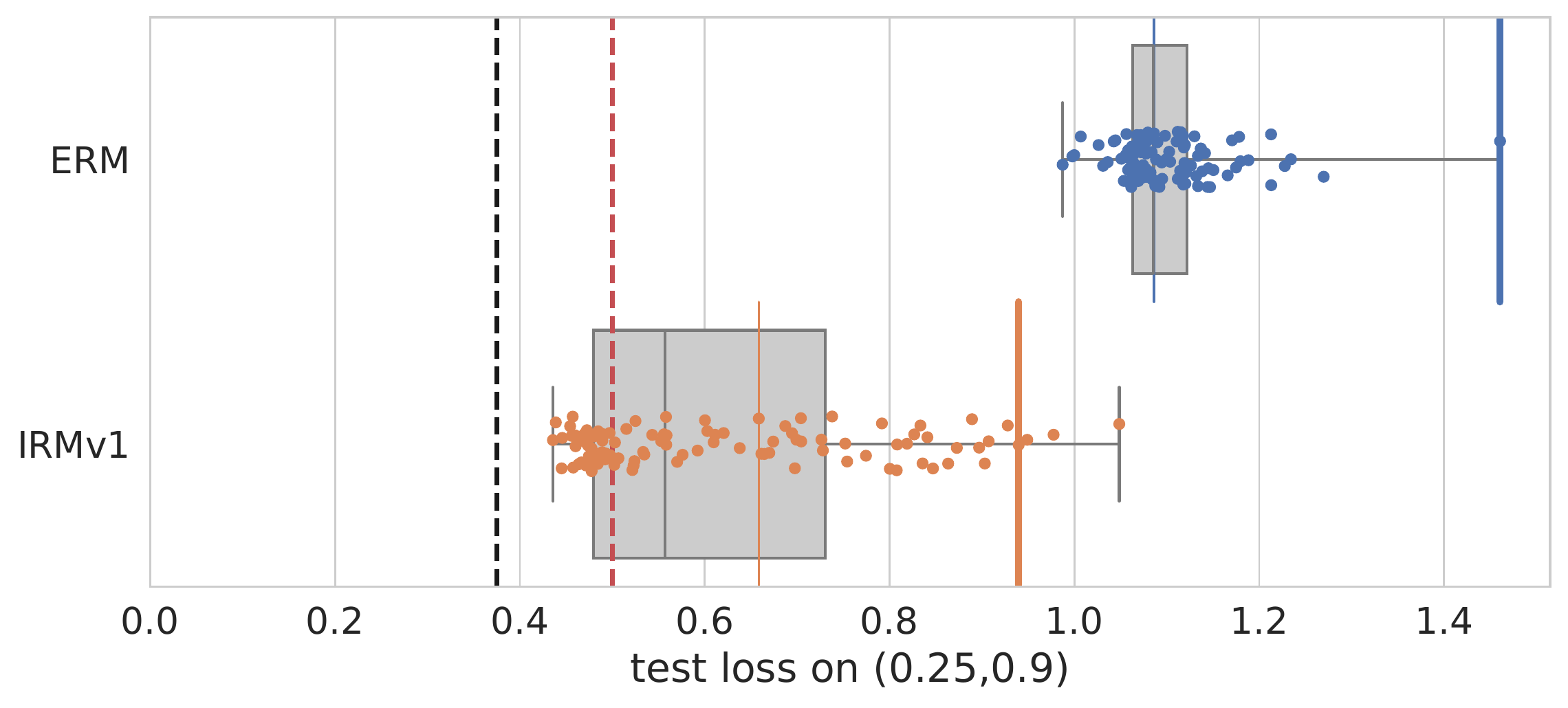}
        \caption{$\Etr = \{ (0.25, 0.1), (0.25, 0.2) \}$, $e_\mathrm{test} = (0.25, 0.9)$}
        \label{fig:square:normal:digit:good}
    \end{subfigure}

    \begin{subfigure}[b]{\columnwidth}
        \includegraphics[width=\textwidth]{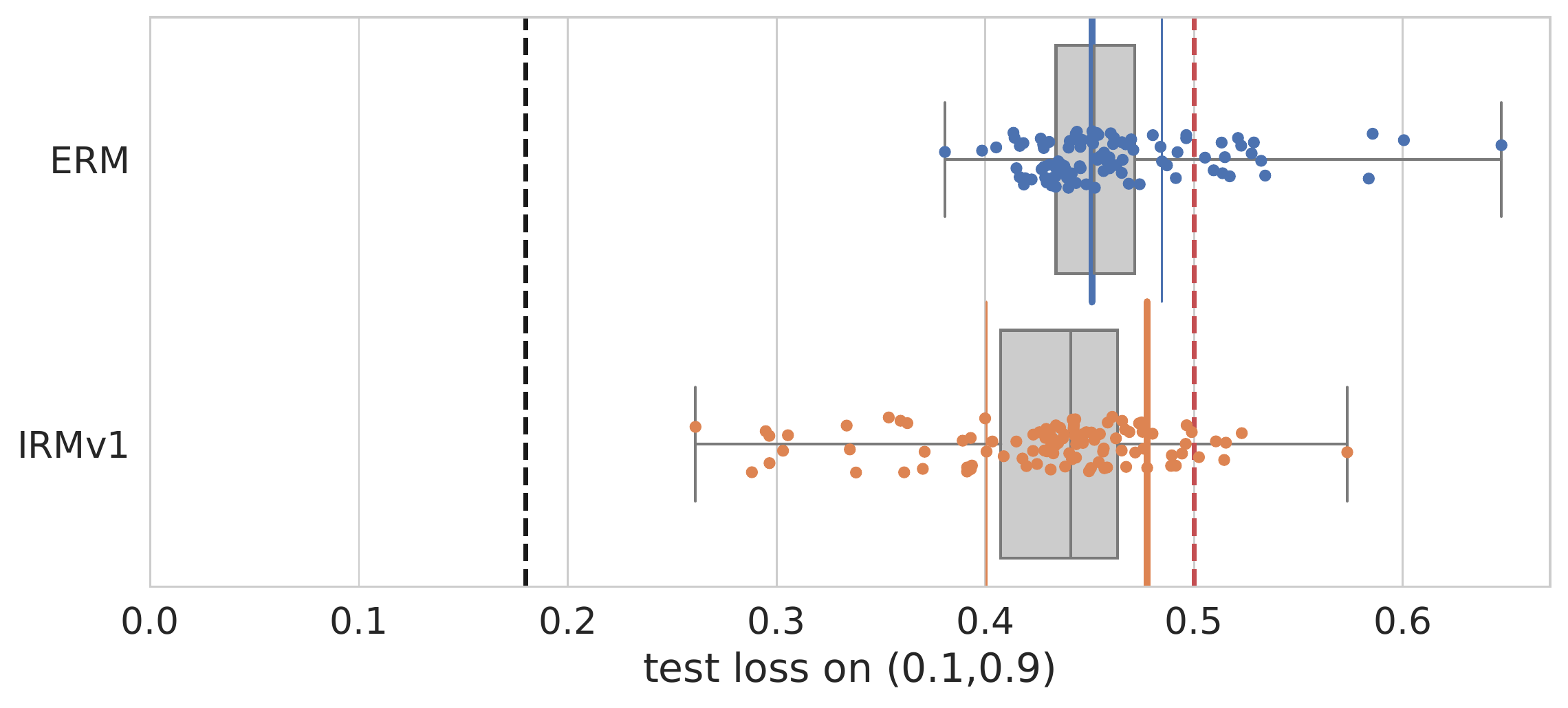}
        \caption{$\Etr = \{ (0.1, 0.2), (0.1, 0.25) \}$, $e_\mathrm{test} = (0.1, 0.9)$}
        \label{fig:square:normal:digit:bad}
    \end{subfigure}

    \caption{%
        Performance on $e_\mathrm{test}$ when training the given algorithm on $\Etr$, using square loss $\sqloss$,
        with a fully-connected network.
        100 repetitions are shown,
        using different random hyperparameters and training splits;
        boxplots show sample quartiles.
        Black dashed line (left) shows expected loss of the optimal invariant predictor $\pred_\IRM$;
        red dashed line (right) shows expected loss of the other predictor in $\calI(\Etr)$, the null predictor $\pred_0$.
        Shorter, colored vertical lines show the test set performance of the predictor
        which minimizes the training objective, \eqref{eqn:erm} or \eqref{eqn:irmv1} (with $\lambda = 10^6$).
    }
    \label{fig:square:normal}
\end{figure}

One significant practical issue with $\IRMvone$ is in hyperparameter tuning,
since we wish to find models which generalize to environments quite different from $\Etr$.
\Citet{arjovsky19invariant} chose hyperparameters
arbitrarily for their $\ERM$ networks,
and for $\IRMvone$ by selecting a network with randomly selected hyperparameters
which performed the best \emph{on the test set} (specifically, the model with the highest minimum accuracy on $\Etr \cup \{ e_\mathrm{test} \}$).
Since this significantly advantages $\IRMvone$ over $\ERM$,
we instead consider the distribution of performances with random hyperparameters
from the same proposal distribution as used by \citeauthor{arjovsky19invariant}.
We also note which of these models minimized the objective on $\Etr$
(using a fixed, large $\lambda$ to compare the objective for $\IRMvone$).
Currently, there is no known principled approach for choosing $\lambda$; as noted by \citet{gulrajani20lost}, this is often critical to the practical performance of $\IRM$.

\citet{arjovsky19invariant} use a fully-connected ReLU network with one hidden layer,
operating on the red and green channels of a $14 \times 14$ image.
Running $\ERM$ and $\IRMvone$ on this architecture with $\sqloss$
in the original \ColoredMNIST problem
shows  (\cref{fig:square:normal:digit:good})
that $\IRMvone$ handily outperforms $\ERM$ in test loss,
though it does not quite achieve the performance of the best possible $\pred_\IRM$,
and model selection based on $\Etr$ would choose a predictor notably worse on the test set than the null predictor $\pred_0$.
Moving to the example failure mode discussed above, this is no longer the case (\cref{fig:square:normal:digit:bad}):
the two algorithms perform about the same in test loss,
with model selection on $\Etr$ selecting a model with performance about the same as $\pred_0$ for each algorithm.
Although the practical instantiation of $\IRMvone$ clearly suffers here,
it is not \emph{worse} than $\ERM$ as we would expect for the population-optimal solutions.

In this representation, $X_1$ (digit) and $X_2$ (color) are quite ``entangled.''
In \cref{appendix:colormnist}, we consider an architecture which processes
the grayscale image and total color of the image separately,
thus becoming a little closer to the idealized setting \eqref{eqn:two-bit-envs};
here the failure of $\IRMvone$ compared to $\ERM$ becomes more apparent.
We also explore many variations of the experiment,
including experiments with $\logloss$.

Thus, $\IRM_{\Wscalar}$'s surprising failure on the extremely simple problem \eqref{eqn:two-bit-envs}
is essentially reproduced with practical optimization of neural networks on $\ColoredMNIST$.

\section{\boldmath CAN \texorpdfstring{$\IRM$}{IRM} FAIL TO CHOOSE THE RIGHT PREDICTOR?} \label{sec:invariance-sufficient}
In the previous section,
we saw an example where
$\IRM_\Wscalar$ was able to identify $\Iscalar(\calE)$,
since $\Iscalar(\calE) = \Iscalar(\Etr)$ there,
but chose a predictor in $\Iscalar(\calE)$
with worse out-of-distribution risk
for environments ``far from'' $\Etr$. This happened because the loss $\calL_e(\pred)$ of predictors $\pred \in \Iscalar(\calE)$ need not be the same (invariant) for all environments $e \in \calE$, and we pick the ``wrong'' predictor when optimizing $\sum_{e \in \Etr} \calL_{e}(\pred)$ over $\pred \in \Iscalar(\Etr)$.

Is the same possible for $\IRM$,
or does its implicit premise that the optimal \emph{invariant} predictor
on $\Etr$ will generalize well to $\calE$ hold? $\IRM$ can of course fail when $\calI(\Etr) \supsetneq \calI(\calE)$, when the training environments are not diverse enough to identify the right invariances. But what if we do have $\calI(\Etr) = \calI(\calE)$?

The loss of an invariant predictor $\pred \in \calI(\calE)$ need not be invariant for all $e \in \calE$: consider e.g.\ varying amounts of inherent additive noise in a regression setting.
This would still be acceptable as long as the \emph{best} invariant predictor with respect to the population loss is the same for all environments $e \in \calE$.
Contrarily, we now give a simple family of environments $\calE$,
training environments $\Etr \subseteq \calE$ satisfying $\calI(\Etr) = \calI(\calE)$,
and two predictors $\pred_1, \pred_2 \in \calI(\calE)$
such that $\calL_{e}(\pred_1) > \calL_{e}(\pred_2)$ for all $e \in \Etr$,
but $\calL_{\calE}(\pred_1) < \calL_{\calE}(\pred_2)$.
Hence $\IRM$ prefers $\pred_2$ to $\pred_1$ based on $\Etr$,
but $\pred_1$ has better worst-case loss.
It is thus generally difficult to handle out-of-distribution prediction in environments with more than one invariant predictor:
the invariant predictor which is best on training environments
might still perform poorly on unseen test environments,
\emph{despite being invariant}.

Consider environments $\calE$ over $\calX = \set{-1,0,1}^3$ and $\calY = \sbit$, where each environment $e$ is specified by a single parameter $\theta_e \in (-1/6, 1/3)$ as follows:
\begin{gather*}
    X_1 \gets \infork{
        -1 & \text{w.p. } \frac{1}{3}\\[1mm]
        \phantom{-}0 & \text{w.p. } \frac{1}{3}\\[1mm]
        +1 & \text{w.p. } \frac{1}{3}
    },\ 
    X_2 \gets \infork{
        -1 & \text{w.p. } \frac{1}{3}-\phantom{2}\theta_e\\[1mm]
        \phantom{-}0 & \text{w.p. } \frac{1}{3}+2\theta_e\\[1mm]
        +1 & \text{w.p. } \frac{1}{3}-\phantom{2}\theta_e
    }, \\[2mm]
    \Ex_{\calD_e}[Y | X_1, X_2] = 0.3 (X_1 + X_2) + g_{\theta_e}(X_1, X_2)\ ,
\end{gather*}
where $g_{\theta_e}(x_1, x_2)$ is given as\vspace{-4mm}
\begin{center}{\renewcommand{\arraystretch}{1.4}
\begin{tabular}{|l!{\vrule width 1.3pt}c|c|c|}
	\hline
	$g_{\theta}(x_1, x_2)$ & $x_2=-1$ & $x_2=0$ & $x_2=+1$\\
	\noalign{\hrule height 1.3pt}
	$x_1 = -1$ & $\theta(\theta + \frac{2}{3})$ & $-\theta(\frac{2}{3} - 2\theta)$ & $3\theta^2$\\
	\hline
	$x_1 = 0$ & $-\theta(\frac{2}{3} - 2\theta)$ & 0 & $\theta(\frac{2}{3} - 2\theta)$ \\
	\hline
	$x_1 = +1$ & $-3\theta^2$ & $\theta(\frac{2}{3} - 2\theta)$ & $- \theta(\theta + \frac{2}{3})$\\
	\hline
\end{tabular}.}
\end{center}
While the specific form of $g_{\theta}$ is a little involved, the main thing to note is that
\[ \Ex_{\calD_e}[g_{\theta_e}(X_1, X_2) \mid X_1] = 0 = \Ex_{\calD_e}[g_{\theta_e}(X_1, X_2) \mid X_2]\, \]
which means that $\Ex_{\calD_e}[Y | X_1] = 0.3 X_1$ as well as $\Ex_{\calD_e}[Y | X_2] = 0.3 X_2$.
Thus for $\sqloss$, $\calI(\calE)$ contains the predictors $\pred_1(x) = 0.3\, x_1$ and $\pred_2(x) = 0.3\, x_2$.
In fact, as shown in \cref{appendix:inv-suff-details}, $\IRM$ will indeed pick among these predictors in $\calI(\calE)$ for almost all $\Etr$ containing at least two distinct environments:

\begin{restatable}{proposition}{propinvinsuff}\label{prop:inv-insuff}
In \cref{setting:binary}, for $\calE$ as above, it holds for Lebesgue-almost all $\Etr \subseteq \calE$ with $|\Etr| \ge 2$ that
$ \calI(\calE) = \calI(\Etr) $.
Moreover, any $\pred \in \calI(\calE)$ depends on at most one of $x_1$ or $x_2$.
\end{restatable}

Focusing on the case of $\sqloss$, the loss of the predictors can be seen to be\footnote{This calculation does not need the specific form of $g_\theta$.}, for any $e \in \calE$,
\[
\calL_e(\pred_1) = 0.47 \quad \text{ and } \quad \calL_e(\pred_2) = 0.47 + 0.09 \cdot \theta_e
\]
Thus, if $\Etr$ only contains environments $e$ corresponding to $\theta_e < 0$, we will have that $\calL_e(\pred_2) < \calL_e(\pred_1)$ for all $e \in \Etr$, and yet the invariant predictor that minimizes $\sup_{e \in \calE} \calL_{e}(\cdot)$ is $\pred_1$. See \cref{fig:inv-insuff-losses} (\cref{appendix:inv-suff-details}) for an illustration of these loss as a function of $\theta_e$.

$\IRM$'s notion of invariance ensures $\Ex_{\calD_e}[Y \mid \varphi(X) = z]$ is invariant across $\calE$,
but allows the loss of the corresponding predictor $\calL_{e}(\pred)$ to differ across $e \in \calE$.
Here, in fact the full conditional distribution $\set{Y \mid \varphi(X)=z}$ is also invariant across $\calE$,
but even so, the loss varies.
\info{What are other papers that consider this? Risks of IRM cites like 6, but they don't seem relevant.}%
If we enforced a stronger notion of invariance
which requires the entire joint distribution $\set{(Y, \varphi(X))}_{(X,Y) \sim \calD_e}$ to be invariant across all $e \in \calE$,
we would not have faced this issue,
since $\calL_e$ would then be invariant, and indeed would pick $\pred_1$ in the example above.
Yet this joint invariance is clearly too strict for some problems:
it is impossible to achieve if the marginal distribution of $Y$ differs across environments,
and it is easy to construct other $\calE$ where $\IRM$ allows the intuitively-correct predictor
but joint invariance allows only a trivial constant predictor.

Thus, $\IRM$ is not always guaranteed to achieve optimal out-of-distribution loss, even when all the right invariances are captured by the training environments. The ``right'' notion of invariance really depends on what we know about the set of all environments $\calE$.

\section{WHEN DOES INVARIANCE GENERALIZE?}\label{sec:invariance-generalize}
In the examples of \cref{sec:bit-example,sec:invariance-sufficient}, it held that $\IRM$ or $\IRM_\Wscalar$ were able to identify predictors invariant over all, even unseen, environments:
specifically, $\calI_\calW(\calE) = \calI_\calW(\Etr)$.
That this holds is an implicit premise of the $\IRM$ framework.
Yet it is unclear in general when invariances discovered on training environments will generalize to unseen environments.
We now give some partial answers to this question.

For an arbitrary $\calE$, we of course cannot expect invariances observed across $\Etr$ to generalize over $\calE$:
simply consider adding a single entirely ``irrelevant'' $e$ to $\calE$.
To provide some structure, we consider parameterized sets of environments $\calE$.
For simplicity, we focus on finite $\calX$ and $\calY$, with $\calY \subseteq \bbR$.
Let $\Delta_{\calX \times \calY}$ denote the space of all probability distributions over $\calX \times \calY$, and let $\Theta \subseteq \bbR^d$.
A map $\Pi : \Theta \to \Delta_{\calX \times \calY}$ naturally defines a set of environments $\calE_{\Pi}$ corresponding to the set of distributions $\set{\Pi(\theta) : \theta \in \Theta}$.
For example, the two-bit environments $\calE_{\alpha}$ of \cref{sec:bit-example} are parameterized by the map $\Pi : \theta \mapsto e = (\alpha, \theta)$, for $\theta \in \Theta = (0,1)$.
\begin{center}
\em For $\Thetatr \subseteq \Theta$ and $\Etr = \set{ \Pi(\theta) \mid \theta \in \Thetatr }$, \\ when does it hold that $\calI(\Etr) = \calI(\calE_{\Pi})$?
\end{center}
Note that $\calI(\calE_\Pi) \subseteq \calI(\Etr)$ always holds,
but for any hope of $\calI(\Etr) \subseteq \calI(\calE_{\Pi})$, we must assume $\Etr$ contains a ``representative set'' of environments from $\calE_{\Pi}$.%

The most basic assumption to begin with is simply that $\Pi$ is continuous.
This is insufficient to guarantee invariance,
even for very large $\Thetatr$:
the map might simply ``change directions'' outside of $\Etr$.
We give a simple example below (proof in \cref{appendix:analytic-invariance}),
where even an uncountable number of environments in $\Etr$ do not allow us to understand the full behavior of $\calE$.
\begin{restatable}{proposition}{propcontinuousinvariance}\label{prop:continuous-invariance}
There exists a continuous map $\Pi : (0,1) \to \Delta_{\calX \times \calY}$ such that for $\Thetatr = \inparen{ 0,\frac14}$ and $\Etr = \Pi(\Thetatr)$, it holds that $\calI(\Etr) \ne \calI(\calE_{\Pi})$.
\end{restatable}
On the other hand, if $\Pi$ is not only continuous but also analytic,
we {\em can} guarantee, under some conditions, that
invariances over $\Etr$ continue to hold over all of $\calE_\Pi$.
Let $\Pi_{(x, y)}(\theta) := \Pr_{(X, Y) \sim \Pi(\theta)}[X=x, Y=y]$
for each $(x, y) \in \calX \times \calY$. We say the map $\Pi : \Theta \to \Delta_{\calX \times \calY}$ is {\em analytic}
if, for each $(x, y) \in \calX \times \calY$,
$\Pi_{(x,y)} : \Theta \to [0,1]$ is analytic in $\theta$.

\begin{restatable}{proposition}{propanalyticinvariance}\label{prop:analytic-invariance}
Let $\Thetatr \subseteq \Theta \subseteq \bbR^d$, where $\Theta$ is a connected, open set.
Suppose $\Pi : \Theta \to \Delta_{\calX \times \calY}$ is analytic, $\calX$ and $\calY$ are finite and $\Etr = \Pi(\Thetatr)$. Then, under \cref{setting:binary},
\begin{enumerate}[leftmargin=6mm,label={(\roman*)}]
\item For almost all $\Thetatr$ with $|\Thetatr| \ge 2$: $\calI(\Etr) = \calI(\calE_{\Pi})$.
\item For all $\Thetatr$ with non-zero Lebesgue measure: $\Iscalar(\Etr) = \Iscalar(\calE_{\Pi})$.
\end{enumerate}
\end{restatable}

The key step is that when $\Pi$ is analytic, the conditional expectations $\Ex_{\Pi(\theta)} [Y \mid \varphi(X) = z]$ and the gradient $\nabla_{w|w=1} \calL_{\Pi(\theta)}(w \cdot \varphi)$ are analytic functions in $\theta$;
the result is far stronger, however, for $\calI$ (where the set of representations is finite)
than for $\Iscalar$, where our analysis requires uncountably many training environments.
A version of \cref{prop:analytic-invariance} holds even for infinite spaces $\calX$ and $\calY \subseteq \bbR$, under a technical definition of analyticity of $\Pi$ (details in \cref{appendix:analytic-invariance}), although in this case our result for $\calI$ also requires $\Etr$ to have positive measure.

Recall that the examples studied in \cref{sec:bit-example,sec:invariance-sufficient} indeed had analytic parameterizations, and hence \cref{prop:analytic-invariance} implies that $\calI(\Etr) = \calI(\calE)$ holds for (almost) all $\Etr$ with at least two distinct environments.

\section{IRM WITH FINITE SAMPLES}\label{sec:sampling-issues}

Except for \cref{sec:colored-mnist}, we have so far only discussed algorithms ($\IRM$, $\IRM_\Wscalar$, and $\IRMvone$) defined in terms of the {\em population} losses of training environments.
In practice, however, we need to work with a finite number of samples from each training environment.
If we directly apply $\IRM$ or $\IRM_{\Wscalar}$ as stated in \eqref{eqn:irm} to empirical distributions,
all correlations will have a small amount of noise,
and it is extremely likely that the set of invariant predictors becomes empty.

On the other hand, $\IRMvone$ for a fixed $\lambda$ could be robust to sampling. We illustrate this in the two-bit environments of \cref{sec:bit-example}.
Consider training environments $\Etr = \set{(0.25,0.1), (0.25,0.2), (0.25,0.3)}$:
both $\IRM$ and $\IRM_{\Wscalar}$ are able to learn an invariant predictor.
However, when sampling finite datasets,
we only have that the empirical distribution of the two environments will be \emph{close} to -- but not exactly the same as -- the true distribution;
there may not be \emph{any} exactly-invariant predictors.
We illustrate this by evaluating $\IRM_{\Wscalar}$ on a set of training environments $\Etr' = \set{(0.245,0.105),(0.255,0.195),(0.251,0.302)}$, as a proxy for empirical distributions we see from finite samples.
$\IRM_{\Wscalar}$ learns the trivial $0$ predictor $\pred_0$;
\cref{fig:irmv1_good_vs_perturbed} shows the behavior of $\IRMvone$ for increasing $\lambda$.

For a fixed empirical distribution, it is likely that as $\lambda \to \infty$, $\IRMvone$ approaches $\IRM_{\Wscalar}$, and does not find a good invariant predictor.
If we instead take $n \to \infty$ for a fixed $\lambda$, though,
we should approach the population version of $\IRMvone$,
and hence %
taking $\lambda \to \infty$ at an appropriate rate as $n \to \infty$
may approach the population $\IRM_\Wscalar$ predictor.
\Citet{ahuja20sample} recently considered a variant of $\IRM_{\Wscalar}$ where the constraints \eqref{eq:grad-0-scalar} defining $\Iscalar(\Etr)$ need to hold $\eps$-approximately. When training on the objective with finite samples, they bounds the sample complexity to get an out-of-distribution loss close to that of the corresponding population version of this $\eps$-$\IRM_{\Wscalar}$.

\begin{figure}[t]
\centering
\begin{lpic}[r(7mm)]{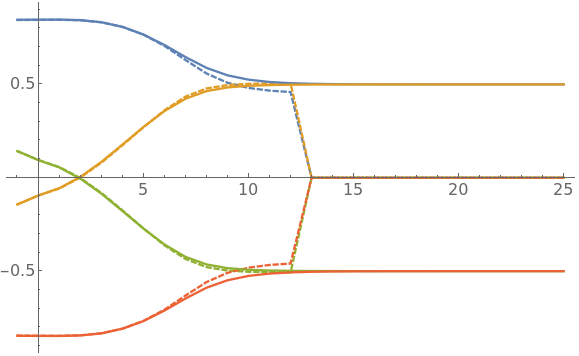(0.46)}
\scriptsize
\lbl[c]{158,45;\textcolor{black!60}{\boldmath $\log_2(\lambda)$}}
\tiny
\lbl[l]{47,81;\textcolor{WolframOne}{\boldmath$\pred(1,1)$}}
\lbl[l]{43,58;\textcolor{WolframTwo}{\boldmath$\pred(1,-1)$}}
\lbl[l]{43,32;\textcolor{WolframThree}{\boldmath$\pred(-1,1)$}}
\lbl[l]{47,9;\textcolor{WolframFour}{\boldmath$\pred(-1,-1)$}}
\end{lpic}
\caption{$\IRMvone$ algorithm on exact environments $\Etr$ (solid lines), and a noisy set $\Etr'$ (dashed lines; definitions in text). The horizontal axis is $\log_2(\lambda)$, with $-1$ for $\lambda = 0$. Results are similar for small $\lambda$, until the noisy set abruptly gives the 0-predictor.}
\label{fig:irmv1_good_vs_perturbed}
\end{figure}
Given the discrepancy between $\IRM_{\Wscalar}$ and $\IRM$ as pointed out in \cref{sec:bit-example}, however, it is important to make $\IRM$ itself more robust to finite samples. For instance, one possible approach would be to relax the requirement of $w \in \argmin_{\overline{w} : \calZ \to \what{\calY}} \calL_e(\overline{w}\circ \varphi)$ to 
\[ \calL_e(w\circ \varphi) ~\le~ \min_{\overline{w} : \calZ \to \what{\calY}}\ \calL_e(\overline{w}\circ \varphi) + \varepsilon \]
for a suitable $\varepsilon > 0$.
How to practically implement a version of this $\varepsilon$-$\IRM$ remains an open challenge.

\section{DISCUSSION}
The $\IRM$ framework of \citet{arjovsky19invariant} proposes a promising new paradigm of learning,
which attempts to exploit information we usually ignore
to find models robust to even some quite dramatic changes in the input distribution.
We have helped shed light on the applicability of this framework.

We now know that $\IRM_\Wscalar$ and $\IRMvone$ can be surprisingly different from $\IRM$, even on very simple environments.
This emphasizes the importance of finding practical algorithms to approximate $\IRM_\calW$ for some nonlinear class of functions $\calW$.

We also know %
that even for $\IRM$, choosing \emph{among} invariant predictors can also be vital for out-of-domain generalization,
and there exist cases where these algorithms choose the wrong one
for out-of-distribution robustness. This holds even if we insist on a stronger notion of invariance, namely that of the conditional distribution $\set{Y \mid \varphi(X)}_{(X,Y)\sim\calD_e}$.
To truly handle worst-case out-of-distribution generalization,
a stronger notion is needed: for example, it suffices to require invariance of the joint distribution $\set{(Y, \varphi(X))}_{(X,Y) \sim \calD_e}$, but this seems overly stringent.

We also now know more about the possibility of generalizing invariances learned from $\Etr$ to a larger set of environments $\calE$.
With significant structure on $\calE$,
it is possible to ensure $\calI(\calE) = \calI(\Etr)$,
but substantial questions remain as to the situation for $\Iscalar$ or more realistic assumptions on $\calE$.

Finally, we demonstrated that $\IRM$ and even $\IRMvone$ can be surprisingly brittle when run on samples, rather than populations.
Thus more analysis, and perhaps new algorithms, are needed
to realize the promise of this framework in practice.

\clearpage 
\acknowledgments{
The authors would like to thank L\'eon Bottou, Martin Arjovksy, Ishaan Gulrajani, and David Lopez-Paz for useful discussions, particularly the derivation of the form of predictors in \cref{apx:bottou-calc}.

Work was supported in part by NSF BIGDATA award 1546500 and NSF RI award 1764032. Work done while the authors participated in a special quarter on the Theory of Deep Learning sponsored by NSF TRIPOD award 1934843 (IDEAL) and while the first author participated in the Theory of Reinforcement Learning program at the Simons Institute for the Theory of Computing.
}

\printbibliography[title=References]
\newpage
\onecolumn
\appendix

\section{More details on \texorpdfstring{\cref{sec:irm}}{Section \ref{sec:irm}}} \label{appendix:irm}

\subsection{Subtleties involving Definitions \ref{def:all-invariant} and \ref{def:space-invariant}} \label{appendix:def-subtleties}

\Cref{def:all-invariant} implicitly assumes that a minimizer $w \in \argmin_{\overline{w} : \calZ \to \what{\calY}} \calL_e(\overline{w} \circ \varphi)$ exists.
This may not always be the case: for example, if we take logistic loss,
there will be no exact maximizer if the problem under $\varphi$ is separable,
i.e.\ $\set{Y \mid \varphi(X)=z}$ is constant for each $z \in \calZ_{\varphi}^e$.
To handle such cases, \cref{def:all-invariant} can be modified as follows.

\begin{definition}\label{def:all-invariant-inf}
A representation $\varphi : \calX \to \calZ$ is {\em invariant} for a set of environments $\calE$ if for all $\eps > 0$, there exists a $w : \calZ \to \what{\calY}$ such that $w$ is simultaneously $\eps$-optimal on $\varphi$ for all environments $e \in \calE$: that is, we have that
$\calL_e(w \circ \varphi) \le \inf_{\overline{w} : \calZ \to \what{\calY}} \calL_e(\overline{w} \circ \varphi) + \eps$.
\end{definition}

A related problem arises in \cref{def:space-invariant},
where $\calL_e(\overline{w} \circ \varphi)$ may not have a minimizer inside $\calW$.
In addition to the case where the data is separable (and hence we would want $w$ to take values $\pm \infty$),
a similar problem can arise even for square loss
if $\calW$ contains points arbitrarily close to the conditional expectation function but not the conditional expectation function itself;
this can happen, for instance, if $\calW$ is a Gaussian RKHS
and the conditional expectation is $L_2$-integrable but not in the RKHS.
To work around this problem, we can allow $w$ to lie in an appropriate ``closure'' of $\calW$.

\subsection{Proofs of \texorpdfstring{\cref{obs:expectations,lem:irm-full-vs-linear}}{Observation~\ref{obs:expectations} and Lemma~\ref{lem:irm-full-vs-linear}}}
\label{appendix:irm-proofs}

The following observation was made by \citet{arjovsky19invariant}.
We include a proof, for completeness and clarity.

\obsexpectations*
\begin{proof}%
Suppose the representation $\varphi : \calX \to \calZ$ is invariant for $\calE$. That is, there exists a predictor $w : \calZ \to \bbR$ such that $w \in \argmin_{\overline{w} : \calZ \to \bbR} \calL_e(\overline{w} \circ \varphi)$ simultaneously for all environments $e \in \calE$. In other words, for all $e \in \calE$ and $z \in \calZ_{\varphi}^e$, it holds that $w(z) \in \argmin_{\omega \in \bbR} \Ex_{\calD_e} \insquare{\ell(\omega, Y) \mid \varphi(X)=z}$.

First, consider the case of $\ell = \sqloss$. It follows that $w(z) = \Ex_{\calD_e} \insquare{Y \mid \varphi(X)=z}$ for all $e \in \calE$ and $z \in \calZ_{\varphi}^e$. In particular, it holds for all $e_1, e_2 \in \calE$ and $z \in \calZ_{\varphi}^{e_1} \cap \calZ_{\varphi}^{e_2}$ that $\Ex_{\calD_{e_1}} \insquare{Y \mid \varphi(X)=z} = \Ex_{\calD_{e_2}} \insquare{Y \mid \varphi(X)=z} = w(z)$.

Conversely, suppose that $\varphi$ is such that $\Ex_{\calD_{e_1}}[Y \mid \varphi(X) = z] = \Ex_{\calD_{e_2}}[Y \mid \varphi(X) = z]$ for all $z \in \calZ_{\varphi}^{e_1} \cap \calZ_{\varphi}^{e_2}$ and all $e_1, e_2 \in \calE$. Then, $w(z) \coloneqq \Ex_{\calD_e}[Y \mid \varphi(X) = z]$ for any $e$ such that $z \in \calZ_{\varphi}^e$ is well-defined and gives a predictor $w$ that is simultaneously optimal for all environments.

The case of $\logloss$ is handled similarly by noting that the minimizer of $\Ex_{\calD_e} \insquare{\logloss(\omega, Y) \mid \varphi(X)=z}$, given by
\[ \omega ~=~ \log \inparen{\frac{\Pr_{\calD_e}[Y=1 \mid \varphi(X) = z]}{\Pr_{\calD_e}[Y=-1 \mid \varphi(X) = z]}} ~=~ \log \inparen{\frac{1 + \Ex_{\calD_e}[Y \mid \varphi(X) = z]}{1 - \Ex_{\calD_e}[Y \mid \varphi(X) = z]}}\,,\]
uniquely corresponds to $\Ex_{\calD_e}[Y \mid \varphi(X) = z]$.
\end{proof}

The following lemma is implicit in \cite{arjovsky19invariant}.
\lemirmfullvslinear*
\begin{proof}%
We prove the lemma in the following three parts.
\begin{description}
\item [\boldmath $\calI(\calE) \subseteq \Iscalar(\calE)$.]
Given $\pred \in \calI(\calE)$, let $\varphi : \calX \to \calZ$ and $w : \calZ \to \bbR$ be such that $\pred = w \circ \varphi$, where $w \in \argmin_{\overline{w} : \calZ \to \bbR} \calL_e(\overline{w} \circ \varphi)$ for all $e \in \calE$.
Define $\varphi':\calX \to \bbR$ as $\varphi'(x) \coloneqq w(\varphi(x))$ and $w': \bbR \to \bbR$ to be the identity function $w'(z) = z$. Thus, we have $w' \circ \varphi' = w \circ \varphi = \pred$. Additionally, it holds that for all $e \in \calE$,  $w' \in \argmin_{\overline{w}' \in \Wscalar}\calL_e(\overline{w}' \circ \varphi')$. (Suppose for contradiction that this is not the case. Then for some environment $e \in \calE$, there exists $c \neq 1$ such that $\calL_e(c \pred) < \calL_e(\pred)$, corresponding to $\overline{w}' \in \Wscalar$ such that $\overline{w}'(z) \coloneqq cz$. Hence $\calL_e((c\cdot w)\circ\varphi) < \calL_e(w \circ \varphi)$, which contradicts that $w \in \argmin_{\overline{w} \in \calW} \calL_e(\overline{w} \circ \varphi)$.)
Thus, we get $\pred \in \Iscalar(\calE)$.

\item [\boldmath $\calI_{\Wlin^d}(\calE) \subseteq \Iscalar(\calE)$.] The proof of the above part shows, more generally, that $\calI_{\calW}(\calE) \subseteq \Iscalar(\calE)$ for any $\calW$ that is closed under scalar multiplications (that is, $w \in \calW \implies c \cdot w \in \calW$ for all $c \in \bbR$), it holds that $\calI_{\calW}(\calE) \subseteq \Iscalar(\calE)$. Since $\Wlin^d$ is closed under scalar multiplications, we get $\calI_{\Wlin^d}(\calE) \subseteq \Iscalar(\calE)$.

\item[\boldmath $\Iscalar(\calE) \subseteq \calI_{\Wlin^d}(\calE)$.] Given $\pred \in \Iscalar(\calE)$, let $\varphi : \calX \to \bbR$ and $w : \bbR \to \bbR$ be such that $\pred = w \circ \varphi$, where $w \in \argmin_{\overline{w} \in \Wscalar} \calL_e(\overline{w} \circ \varphi)$ for all $e \in \calE$. Define $\varphi': \calX \to \bbR^d$ as $\varphi'(x) \coloneqq \varphi(x) \cdot v$ for any unit vector $v \in \bbR^d$ and $w': \bbR^d \to \bbR$ as $w'(z) \coloneqq w(\inangle{z,v})$. It is easy to see that $w' \in \argmin_{\overline{w}' \in \Wlin^d}(\overline{w}' \circ \varphi')$ for all $e \in \calE$ and hence $w' \circ \varphi' = w \circ \varphi = \pred$. Thus, $\pred \in \calI_{\Wlin^d}(\calE)$.\qedhere
\end{description}
\end{proof}

\section{More details on Two-Bit Environments (from \texorpdfstring{\cref{sec:bit-example}}{Section \ref{sec:bit-example}})} \label{appendix:two-bits}

We show that for all $\alpha \in (0,1)$, just two environments in $\calE_{\alpha}$ are sufficient to determine both $\Iscalar(\calE_{\alpha})$ and $\calI(\calE_{\alpha})$. Thus, the failure of $\IRM_{\Wscalar}$ observed in \cref{sec:bit-example} is not due to lack of sufficiently representative training environments, but instead due to the difference between what $\IRM_{\Wscalar}$ deems an ``invariant predictor'' and the notion of invariance as in \cref{def:all-invariant}.

\propidentifyinvariances*
\begin{proof}
{\em (i).} By definition, we have $\Iscalar(\calE_{\alpha}) \subseteq \Iscalar(\Etr)$, since $\Etr \subseteq \calE_{\alpha}$. We show the converse. As noted in \cref{sec:irm}, for any convex and differentiable loss $\ell$ and for any set of environments $\calE$ we have that $\pred \in \Iscalar(\calE)$ if and only if $\pred = 1 \cdot \varphi$ such that $\nabla_{w | w=1} \calL_e(w \cdot \varphi) = 0$ for all $e \in \calE$. The key observation is that for environment $e = (\alpha, \beta_e)$,
\begin{align*}
\nabla_{w | w=1}\ \calL_e(w \cdot \varphi) &~\coloneqq~ \sum_{x_1, x_2, y} \  {\textstyle \Pr_{\calD_e}}[X_1=x_1, X_2=x_2, Y=y] \cdot \nabla_{w | w=1} \ell(w \cdot \varphi(x_1, x_2), y)\\
&~=~ \sum_{x_1, x_2, y} \  ((1-\alpha)\indicator_{x_1=y} + \alpha\indicator_{x_1 \ne y}) \cdot ((1-\beta_e)\indicator_{x_2=y} + \beta_e\indicator_{x_2 \ne y}) \cdot \nabla_{w | w=1} \ell(w \cdot \varphi(x_1, x_2), y)
\end{align*}
is affine in $\beta_e$. In particular, it can be decomposed as $\nabla_{w | w=1}\ \calL_e(w \cdot \varphi) = F(\varphi) + \beta_e G(\varphi)$ for some functions $F$ and $G$. If $\pred \in \Etr$, then we have that $\pred = 1\cdot \varphi$ such that that both $F(\varphi) + \beta_{e_1}G(\varphi) = 0$ and $F(\varphi) + \beta_{e_2}G(\varphi) = 0$ hold, which happens if and only if $F(\varphi) = 0 = G(\varphi)$. This implies $F(\varphi) + \beta_e G(\varphi) = 0$ for all $\beta_e \in (0,1)$, and hence $\pred \in \calE_{\alpha}$.

{\em (ii).} By definition, we have that $\calI(\calE_{\alpha}) \subseteq \calI(\Etr)$, since $\Etr \subseteq \calE_{\alpha}$. We show the converse by establishing that the only invariant predictors in $\calI(\Etr)$ are those that do not depend on $X_2$. By \cref{obs:expectations}, we have that $\varphi$ is invariant over $\Etr$ if and only if $\Ex_{\calD_{e_1}}[Y \mid \varphi(X) = z] = \Ex_{\calD_{e_2}}[Y \mid \varphi(X) = z]$ for all $z \in \calZ_{\varphi}^{e_1} \cap \calZ_{\varphi}^{e_2}$. In other words,  $\varphi$ is invariant over $\Etr$ if and only if $\Ex_{\calD_e}[Y \mid (X_1, X_2) \in \varphi^{-1}(z)]$ is identical for $e_1$ and $e_2$ as long as $\Pr_{\calD_e}[\varphi(X_1, X_2) =z]$ is non-zero in both environments.

\newcommand{\invCol}[1]{\textcolor{og}{\bf #1}}
\newcommand{\noninvCol}[1]{\textcolor{Gred}{#1}}
\begin{table}[ht]
\begin{center}{\renewcommand{\arraystretch}{1.5}
\begin{tabular}{|cccc|c|c|}
\hline
\multicolumn{4}{|c|}{Subset $S \subseteq \sbit^2$} & $\Ex_{\calD_{e}}[Y | (X_1, X_2) \in S]$ & Independent of $\beta_e$?\\
\hline
(1,1) & & & & $\frac{1 -\alpha - \beta_e}{1 - \alpha + (2\alpha-1) \beta_e}$ & \noninvCol{No}\\
 & (1,-1) & & & $\frac{\alpha - \beta_e}{- \alpha + (2\alpha-1)\beta_e}$ & \noninvCol{No}\\
& & (-1,1) & & $\frac{- \alpha + \beta_e}{- \alpha + (2\alpha-1)\beta_e}$ & \noninvCol{No}\\
& & & (-1,-1) & $\frac{\alpha - 1 + \beta_e}{1 - \alpha + (2\alpha-1) \beta_e}$ & \noninvCol{No}\\
(1,1) & (1,-1) & & & $1-2\alpha$ & \invCol{Yes}\\
(1,1) & & (-1,1) & & $1-2\beta_e$ & \noninvCol{No}\\
(1,1) & & & (-1,-1) & $0$ & \invCol{Yes}\\
& (1,-1) & (-1,1) & & $0$ & \invCol{Yes}\\
& (1,-1) & & (-1,-1) & $2\beta_e-1$ & \noninvCol{No}\\
& & (-1,1) & (-1,-1) & $2\alpha - 1$ & \invCol{Yes}\\
(1,1) & (1,-1) & (-1,1) & & $\frac{\alpha - 1 + \beta_e}{- 1 - \alpha + (2\alpha-1) \beta_e}$ & \noninvCol{No}\\
(1,1) & (1,-1) & & (-1,-1) & $\frac{- \alpha + \beta_e}{2 - \alpha + (2\alpha-1) \beta_e}$ & \noninvCol{No}\\
(1,1) & & (-1,1) & (-1,-1) & $\frac{\alpha - \beta_e}{2 - \alpha + (2\alpha-1) \beta_e}$ & \noninvCol{No}\\
& (1,-1) & (-1,1) & (-1,-1) & $\frac{1 - \alpha - \beta_e}{- 1 - \alpha + (2\alpha-1) \beta_e}$ & \noninvCol{No}\\
(1,1) & (1,-1) & (-1,1) & (-1,-1) &  $0$ & \invCol{Yes}\\
\hline
\end{tabular}}
\end{center}
\caption{$\Ex_{\calD_e}[Y \mid X \in S]$ for different choices of $S$ in the proof of Part (ii) of \cref{prop:identify-invariances}.}
\label{tab:subsets}
\end{table}

In \cref{tab:subsets}, we compute $\Ex_{\calD_e}[Y \mid (X_1, X_2) \in S]$ for all possible non-empty subsets $S \subseteq \sbit^2$, in terms of the environment parameters $\alpha$ and $\beta_e$ and track which of these depend or do not depend on $\beta_e$. The ones that depend on $\beta_e$ can be seen to be distinct for any two distinct values of $\beta_e$. Thus, the only invariant representations over $\Etr$ are those corresponding to the following partitions.
\begin{itemize}
\item $\set{\set{(1,1), (1,-1), (-1,1), (-1,-1)}}$, that is, $\varphi(x_1, x_2)$ is constant. The predictor $\pred \in \calI(\Etr)$ corresponding to this representation is the identically zero-predictor $\pred_0$ (for both $\sqloss$ and $\logloss$).
\item $\set{\set{(1,1), (1,-1)}, \set{(-1,1), (-1,-1)}}$, that is, $\varphi(1,1) = \varphi(1,-1)$ and $\varphi(-1,1) = \varphi(-1,-1)$, or essentially $\varphi(x_1, x_2) = x_1$. The predictor $\pred \in \calI(\Etr)$ corresponding to this representation is $\pred(x) = (1-2\alpha) \cdot x_1$ (for $\ell = \sqloss$) or $\pred(x) = \log \frac{(1-\alpha)}{\alpha} \cdot x_1$ (for $\ell = \logloss$) --- see proof of \cref{obs:expectations} for reference.
\item $\set{\set{(1,1), (-1,-1)}, \set{(1,-1), (-1,1)}}$, that is, $\varphi(1,1) = \varphi(-1,-1)$ and $\varphi(1,-1) = \varphi(-1,1)$, or essentially $\varphi(x_1, x_2) = x_1\cdot x_2$. While this representation does depend on $x_2$, the predictor $\pred \in \calI(\Etr)$ corresponding to this representation is the identically zero-predictor $\pred_0$ (for both $\sqloss$ and $\logloss$).
\end{itemize}
In all the above cases, we observe that the invariant representations over $\Etr$ are also invariant over $\calE_{\alpha}$ and moreover, the corresponding predictors are simultaneously optimal for all $e \in \calE_{\alpha}$ and hence in $\calI(\calE_{\alpha})$. Thus, we have $\calI(\Etr) \subseteq \calI(\calE_{\alpha})$.
\end{proof}

\subsection{Case of square loss}\label{apx:two-bits-sqloss}
We recall the example described in \cref{sec:bit-example} that demonstrated the difference between $\IRM_{\Wscalar}$ and $\IRM$. We have $\calE = \calE_{0.1}$ and $\Etr = \set{(0.1,0.2), (0.1,0.25)}$. We get from \cref{prop:identify-invariances} that $\Iscalar(\calE) = \Iscalar(\Etr)$, which can be numerically seen to contain (approximately) the following four predictors, by simultaneously solving \eqref{eqn:grad-constraint} for all $e \in \Etr$.
\begin{center}{\renewcommand{\arraystretch}{1.4}
\begin{tabular}{!{\vrule width 1.3pt}r!{\vrule width 1.3pt}c|c!{\vrule width 1.3pt}c|c!{\vrule width 1.3pt}c|c!{\vrule width 1.3pt}c|c!{\vrule width 1.3pt}}
\noalign{\hrule height 1.3pt}
& \multicolumn{2}{c!{\vrule width 1.3pt}}{$\pred_0$}
& \multicolumn{2}{c!{\vrule width 1.3pt}}{$\pred_{\IRM}$}
& \multicolumn{2}{c!{\vrule width 1.3pt}}{$\pred_{1}$}
& \multicolumn{2}{c!{\vrule width 1.3pt}}{$\pred_{2}$} \\
\cline{2-9}
& $X_2 = +1$ & $X_2 = -1$ & $X_2 = +1$ & $X_2 = -1$ & $X_2 = +1$ & $X_2 = -1$ & $X_2 = +1$ & $X_2 = -1$ \\
\hline
$X_1=+1$
& $0$ & $0$
& $\phantom{-}0.8$ & $\phantom{-}0.8$
& $\phantom{-}0.9557$ & $\phantom{-}0.2943$
& $\phantom{-}0.2943$ & $\phantom{-}0.9557$\\
\hline
$X_1=-1$
& $0$ & $0$
& $-0.8$ & $-0.8$
& $-0.2943$ & $-0.9557$
& $-0.9557$ & $-0.2943$ \\
\noalign{\hrule height 1.3pt}
\end{tabular}
}\end{center}

On the other hand, $\calI(\calE_{0.1})$ contains only two of the predictors, namely $\pred_0$ and $\pred_{\IRM}$, the latter being the optimal predictor chosen by $\IRM$ on $\Etr$ --- note that this predictor depends only on $X_1$.

\cref{fig:scalar-invariant-losses} shows the population square losses $\calL_e$ for each of the predictors in $\Iscalar(\calE_{0.1})$ for all $e \in \calE_{0.1}$. It can observed that for $e = (0.1,\beta_e)$ with $\beta_e < 0.28$, it holds that $\calL_e(\pred_1) < \calL_e(\pred_{\IRM})$. Thus, no matter how many training environments are present in $\Etr$, $\IRM_{\Wscalar}$ will choose $\pred_1$ as the optimal predictor as long as $\beta_e < 0.28$ for all $e \in \Etr$. On the other hand, $\IRM$ with just two environments learns the predictor $\pred_{\IRM}$.

We also note that the value $\alpha=0.1$ is not special either. In fact, a similar phenomenon as above is observed in $\calE_{\alpha}$ for any value of $\alpha < 0.1464$ or $\alpha > 0.8536$. The following section explains the meaning of these cutoff values.

\subsubsection{\boldmath Analytic characterization of odd predictors in \texorpdfstring{$\Iscalar(\calE)$}{IS(E)}}\label{apx:bottou-calc}

Following the initial version of this paper (which found these constants only by numerically solving certain quadratic systems),
L\'eon Bottou communicated to us the following clean analysis, which provides a closed-form understanding of these $\IRM_{\Wscalar}$ solutions and the range of $\alpha$ when such examples arise.
We are grateful to L\'eon for allowing us to include his calculations here.

Firstly, observe that for $\calX = \sbit^2$, a representation $\varphi(x_1, x_2)$ is odd if and only if it is {\em linear}, namely, $\varphi(x_1, x_2) := w_1x_1 + w_2x_2$.
Suppose $\Etr$ consists of two environments $e_1 = (\alpha, \beta_1)$ and $e_2 = (\alpha, \beta_2)$. From \eqref{eqn:grad-constraint}, we for any $\pred(x) = 1 \cdot \varphi(x) = w_1x_1 + w_2x_2 \in \Iscalar(\Etr)$ that
\begin{align*}
\Ex_{\calD_{e_1}} (w_1x_1 + w_2x_2 - y) (w_1x_1 + w_2x_2) &~=~ 0\\
\Ex_{\calD_{e_2}} (w_1x_1 + w_2x_2 - y) (w_1x_1 + w_2x_2) &~=~ 0
.\end{align*}
From the definition \eqref{eqn:two-bit-envs}, we have (i) $\Ex_{\calD_i}[x_1^2] = \Ex_{\calD_i}[x_2^2] = 1$, (ii) $\Ex_{\calD_i} [x_1 y] = a$, (iii) $\Ex_{\calD_i} [x_2 y] = b_i$ and (iv) $\Ex_{\calD_i}[x_1 x_2] = ab_i$, where $a := 1-2\alpha$ and $b_i = 1-2\beta_i$ for $i \in \set{1, 2}$. Thus, we get
\begin{align}
w_1^2 + w_2^2 + 2 w_1 w_2 ab_1 &~=~ w_1 a + w_2 b_1,\label{eq:env-1-scalar}\\
w_1^2 + w_2^2 + 2 w_1 w_2 ab_2 &~=~ w_1 a + w_2 b_2.\label{eq:env-2-scalar}
\end{align}
By subtracting and using $b_1 - b_2 \ne 0$, we get
\begin{equation}\label{eq:env-3-scalar}
2w_1 w_2 a = w_2
.\end{equation}
When $w_2 = 0$, we get from \eqref{eq:env-1-scalar} (or \eqref{eq:env-2-scalar}) that either $w_1 = 0$ or $w_1 = a$. %
But when $w_2 \ne 0$, we have from \eqref{eq:env-3-scalar} that $w_1 = 1/2a$. Substituting this in \eqref{eq:env-1-scalar} (or \eqref{eq:env-2-scalar}), we get the two additional solutions given by
\[ w_1 ~=~ \frac{1}{2a} \qquad \text{ and } \qquad w_2 ~=~ \pm \sqrt{\frac12 - \frac1{4a^2}} .\]
Note that these additional solutions (with $w_2 \ne 0$) exist only when $\frac12 - \frac1{4a^2} > 0$, or $(1-2\alpha)^2 > \frac12$. That is,
\[
\alpha ~<~ \frac12 - \frac1{2\sqrt{2}} ~\approx~ 0.1464 \qquad \text{ or } \qquad
\alpha ~>~ \frac12 + \frac1{2\sqrt{2}} ~\approx~ 0.8536
.\]
In this regime, the four odd (or linear) predictors in $\Iscalar(\calE_{\alpha})$ are
\begin{align*}
\pred_0(x) &~=~ 0\\
\pred_{\IRM}(x) &~=~ (1-2\alpha) \cdot x_1\\
\pred_1(x) &~=~ \textstyle \frac{1}{2-4\alpha} \cdot x_1 + \sqrt{\frac12 - \frac{1}{4 (1-2\alpha)^2}} \cdot x_2\\
\pred_2(x) &~=~ \textstyle \frac{1}{2-4\alpha} \cdot x_1 - \sqrt{\frac12 - \frac{1}{4 (1-2\alpha)^2}} \cdot x_2
.\end{align*}

\subsection{Case of logistic loss}\label{apx:two-bits-logloss}

We observe a similar phenomenon with logistic loss as was observed for square loss. We consider $\calE = \calE_{0.05}$ and $\Etr = \set{(0.05,0.1),(0.05,0.2)}$. Again, we get from \cref{prop:identify-invariances} that $\calI_{\Wscalar}(\calE) = \calI_{\Wscalar}(\Etr)$, which can be numerically seen to contain (approximately) the following predictors, by simultaneously solving \cref{eq:grad-0-scalar} for all $e \in \Etr$.

\begin{center}{\renewcommand{\arraystretch}{1.4}
\begin{tabular}{!{\vrule width 1.3pt}r!{\vrule width 1.3pt}c|c!{\vrule width 1.3pt}c|c!{\vrule width 1.3pt}c|c!{\vrule width 1.3pt}c|c!{\vrule width 1.3pt}}
\noalign{\hrule height 1.3pt}
& \multicolumn{2}{c!{\vrule width 1.3pt}}{$\pred_0$}
& \multicolumn{2}{c!{\vrule width 1.3pt}}{$\pred_{\IRM}$}
& \multicolumn{2}{c!{\vrule width 1.3pt}}{$\pred_{1}$}
& \multicolumn{2}{c!{\vrule width 1.3pt}}{$\pred_{2}$} \\
\cline{2-9}
& $X_2 = +1$ & $X_2 = -1$ & $X_2 = +1$ & $X_2 = -1$ & $X_2 = +1$ & $X_2 = -1$ & $X_2 = +1$ & $X_2 = -1$ \\
\hline
$X_1=+1$
& $0$ & $0$
& $\phantom{-}2.9444$ & $\phantom{-}2.9444$
& $\phantom{-}4.9847$ & $\phantom{-}0.9041$
& $\phantom{-}0.9041$ & $\phantom{-}4.9847$\\
\hline
$X_1=-1$
& $0$ & $0$
& $-2.9444$ & $-2.9444$
& $-0.9041$ & $-4.9847$
& $-4.9847$ & $-0.90413$ \\
\noalign{\hrule height 1.3pt}
\end{tabular}
}\end{center}
On the other hand, $\calI(\calE_{0.05})$ contains only two of the predictors, namely $\pred_0$ and $\pred_{\IRM}$, the latter being the optimal predictor chosen by $\IRM$ on $\Etr$ --- note that this predictor depends only on $X_1$.

\begin{figure}[t]
\centering
\begin{lpic}[l(5mm),b(4mm)]{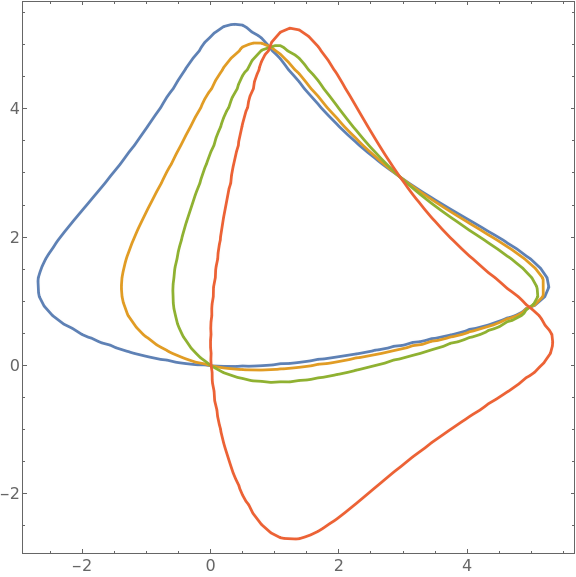(.5)}
\normalsize
\lbl[c]{-6,80,90;\textcolor{black!60}{$\varphi(1,1)=-\varphi(-1,-1)$}}
\lbl[c]{80,-6;\textcolor{black!60}{$\varphi(1,-1)=-\varphi(-1,1)$}}
\tiny
\lbl[c]{31,114,52;\textcolor{WolframOne}{\boldmath$e=(0.05,0.1)$}}
\lbl[c]{37,103,61;\textcolor{WolframTwo}{\boldmath$e=(0.05,0.2)$}}
\lbl[c]{43,90,70;\textcolor{WolframThree}{\boldmath$e=(0.05,0.4)$}}
\lbl[c]{61,88,81;\textcolor{WolframFour}{\boldmath$e=(0.05,0.9)$}}
\scriptsize
\lbl[c]{49,49,0;\textcolor{black!70}{\boldmath$\pred_0$}}
\lbl[c]{107,107,0;\textcolor{black!70}{\boldmath$\pred_{\IRM}$}}
\lbl[c]{70,125,0;\textcolor{black!70}{\boldmath$\pred_1$}}
\lbl[c]{125,70,0;\textcolor{black!70}{\boldmath$\pred_2$}}
\end{lpic}
\caption{Odd solutions to $\nabla_{w | w=1} \calL_e(w \cdot \varphi) = 0$ (with $\ell = \logloss$) for four environments in $\calE_{0.05}$. (Compare to \cref{fig:sqls_grad_constraints}.)}
\label{fig:logls_grad_constraints}
\end{figure}

\begin{figure}[t]
\centering
\begin{lpic}[t(4mm),r(5mm)]{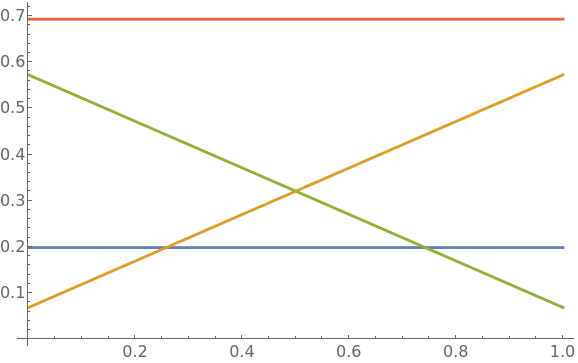(0.5)}
\normalsize
\lbl[c]{10,99;\textcolor{black!60}{\boldmath $\calL_e(\pred)$}}
\lbl[c]{153,6;\textcolor{black!60}{\boldmath $\beta_e$}}
\lbl[c]{75,24;\textcolor{WolframOne}{\boldmath$\pred_{\IRM}$}}
\lbl[c]{27,17;\textcolor{WolframTwo}{\boldmath$\pred_1$}}
\lbl[c]{127,17;\textcolor{WolframThree}{\boldmath$\pred_2$}}
\lbl[c]{75,82;\textcolor{WolframFour}{\boldmath$\pred_0$}}
\end{lpic}
\caption{Losses $\calL_e$ (for $\ell = \logloss$) of odd predictors in $\Iscalar(\calE_{0.05})$ for various $e = (0.05,\beta_e)$. (Compare to \cref{fig:scalar-invariant-losses}.)}
\label{fig:scalar-invariant-losses-logistic}
\end{figure}

\cref{fig:scalar-invariant-losses-logistic} shows the population square losses $\calL_e$ for each of the predictors in $\Iscalar(\calE_{0.05})$ for all $e \in \calE_{0.05}$. It can observed that for $e = (0.05,\beta_e)$ with $\beta_e < 0.25$, it holds that $\calL_e(\pred_1) < \calL_e(\pred_{\IRM})$. Thus, no matter how many training environments are present in $\Etr$, $\IRM_{\Wscalar}$ will choose $\pred_1$ as the optimal predictor as long as $\beta_e < 0.25$ for all $e \in \Etr$. On the other hand, $\IRM$ with just two environments learns the predictor $\pred_{\IRM}$.

We also note that the value $\alpha=0.05$ is not special; a similar phenomenon as above is observed in $\calE_{\alpha}$ for any value of $\alpha < 0.077$.

In the supplementary material, we include the Mathematica code (\texttt{two-bit/two-bit-irm.nb}, and a PDF version \texttt{two-bit/two-bit-irm.pdf}) that was used to compute $\pred_{\IRM}$ and $\pred_{\IRM_{\Wscalar}}$ solutions and plot \cref{fig:sqls_grad_constraints,fig:scalar-invariant-losses,fig:irmv1_interpolation,fig:irmv1_good_vs_perturbed,fig:logls_grad_constraints,fig:scalar-invariant-losses-logistic}.

\section{More \ColoredMNIST experiments} \label{appendix:colormnist}

We now consider more details and variations of the \ColoredMNIST experiments of \cref{sec:colored-mnist}.

The architecture used by \textcite{arjovsky19invariant} is fully connected,
mapping inputs of dimension $2 \cdot 14 \cdot 14$ to hidden dimension $h$,
from $h$ to $h$,
and then from $h$ to a scalar prediction,
with ReLU activations on each layer except the last.
The model is optimized with full-batch Adam for 501 steps,
with a scaled penalty on the squared (Frobenius) norm of each parameter,
and hyperparameters selected as:
\begin{itemize}
\item Hidden dimension $h$: $\lfloor 2^{\mathrm{Uniform}[6, 9)} \rfloor$.
\item Weight of $L_2$ regularization: $10^{\mathrm{Uniform}[-2, -5)}$.
\item Learning rate: $10^{\mathrm{Uniform}[-2.5, -3.5)}$.
\item For $\IRMvone$, the gradient penalty weight $\lambda$ is $1$ for $\mathrm{Uniform}\{50, 51, \dots, 250\}$ iterations, then $10^{\mathrm{Uniform}[2, 6)}$.
\end{itemize}

In \cref{fig:square:normal-full},
we reproduce the results of \cref{fig:square:normal} (left column)
but also show results of versions of the architecture
forced to depend only on $X_1$ or $X_2$
while training via $\ERM$:
\texttt{color-only} takes inputs of shape $2$,
a one-hot indicator for whether the color is red or green,
while \texttt{digit-only} receives a flattened grayscale image of dimension $14 \cdot 14$.
This allows us to see the amount of variation we can expect based purely on changes in the learning process.
We also show (in the right column) a flipped version of the problem,
where the invariant feature is color rather than the digit identity;
this is the same from the point of view of the abstract Two-Bit environment,
but allows us to see how much of the behavior depends on the different way that this network processes digit and color information.

As mentioned in \cref{sec:colored-mnist},
we also consider a ``split'' variant of the architecture,
which is perhaps closer to the abstract two-bit version.
Here,
the network has two branches:
one takes a grayscale $14 \times 14$ version of its input,
which is processed as in the previous architecture down to a scalar.
The other branch takes a one-hot (two-dimensional) indicator for the color,
and (via a $2 \times 1$ linear layer) outputs an arbitrary scalar for each color.
The top of the network takes in these two scalar values,
processes them with an 8-dimensional ReLU layer,
then makes a final linear prediction.
\texttt{color-only} and \texttt{digit-only} versions simply omit one of those branches.
Results for $\sqloss$ are shown in \cref{fig:square:split-full}.
Here we most clearly see the ``average-case'' failure of $\IRMvone$ in the color-invariant case.

Similar results for $\logloss$ are shown in \cref{fig:nll:split-full,fig:nll:normal-full}.
The expected failure mode is generally less visible here,
though it is more evident in the color-invariant settings than the digit-invariant ones.

In the supplementary material, we include the PyTorch code, modified from that of \citet{arjovsky19invariant}, used to produce these results (\texttt{colored-mnist} directory).

\begin{figure}
    \centering
    \includegraphics[width=\textwidth]{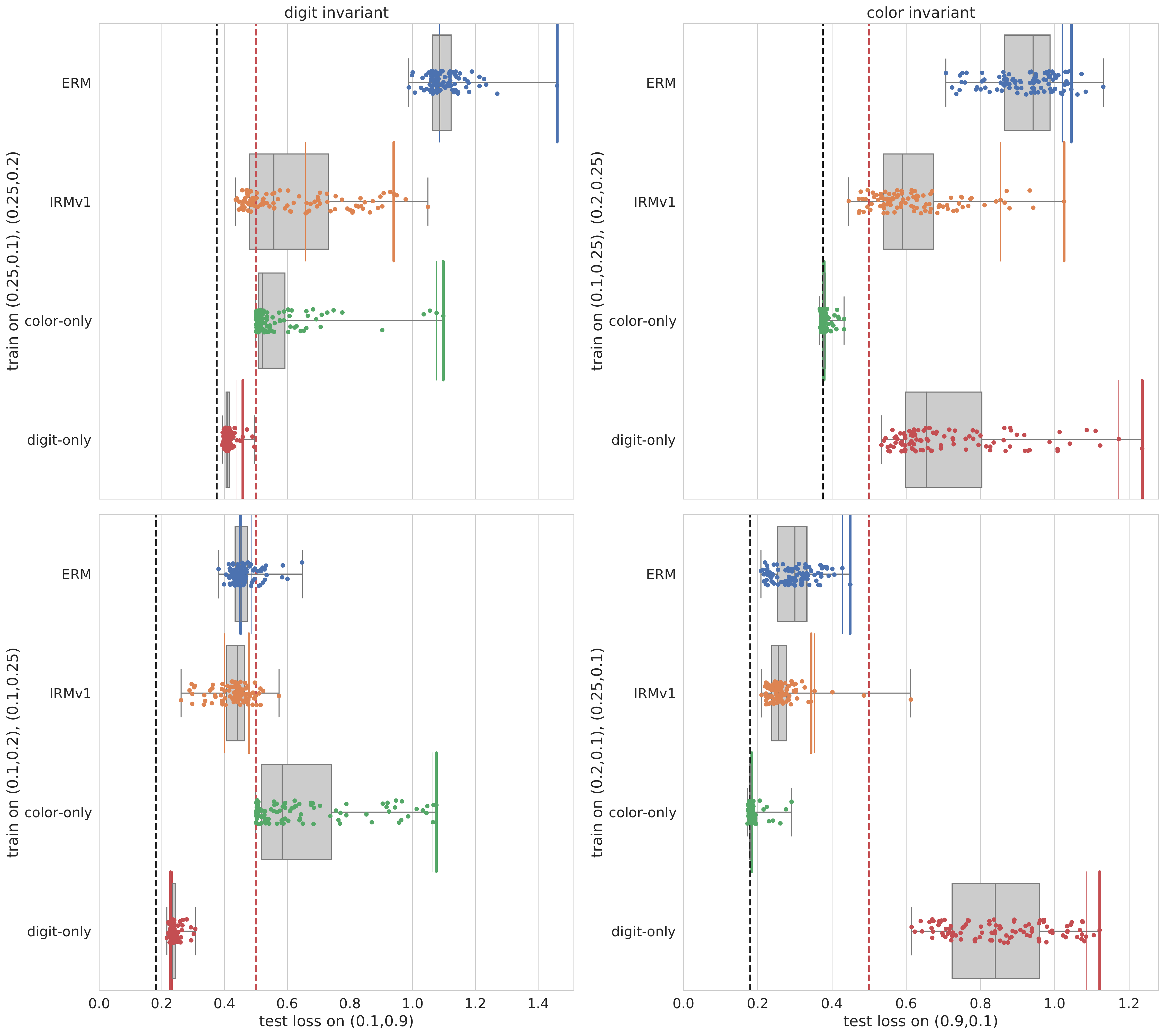} 
    \caption{\ColoredMNIST using $\sqloss$ with the architecture of \citet{arjovsky19invariant}: the same as \cref{fig:square:normal}, but additionally showing cases where color is invariant rather than the digit (right column), and performance of networks which receive only grayscale digits as input, or only a one-hot indicator of the color. Thin colored lines show performance of the second-best hyperparameter setting on the training environments.}
    \label{fig:square:normal-full}
\end{figure}

\begin{figure}
    \centering
    \includegraphics[width=\textwidth]{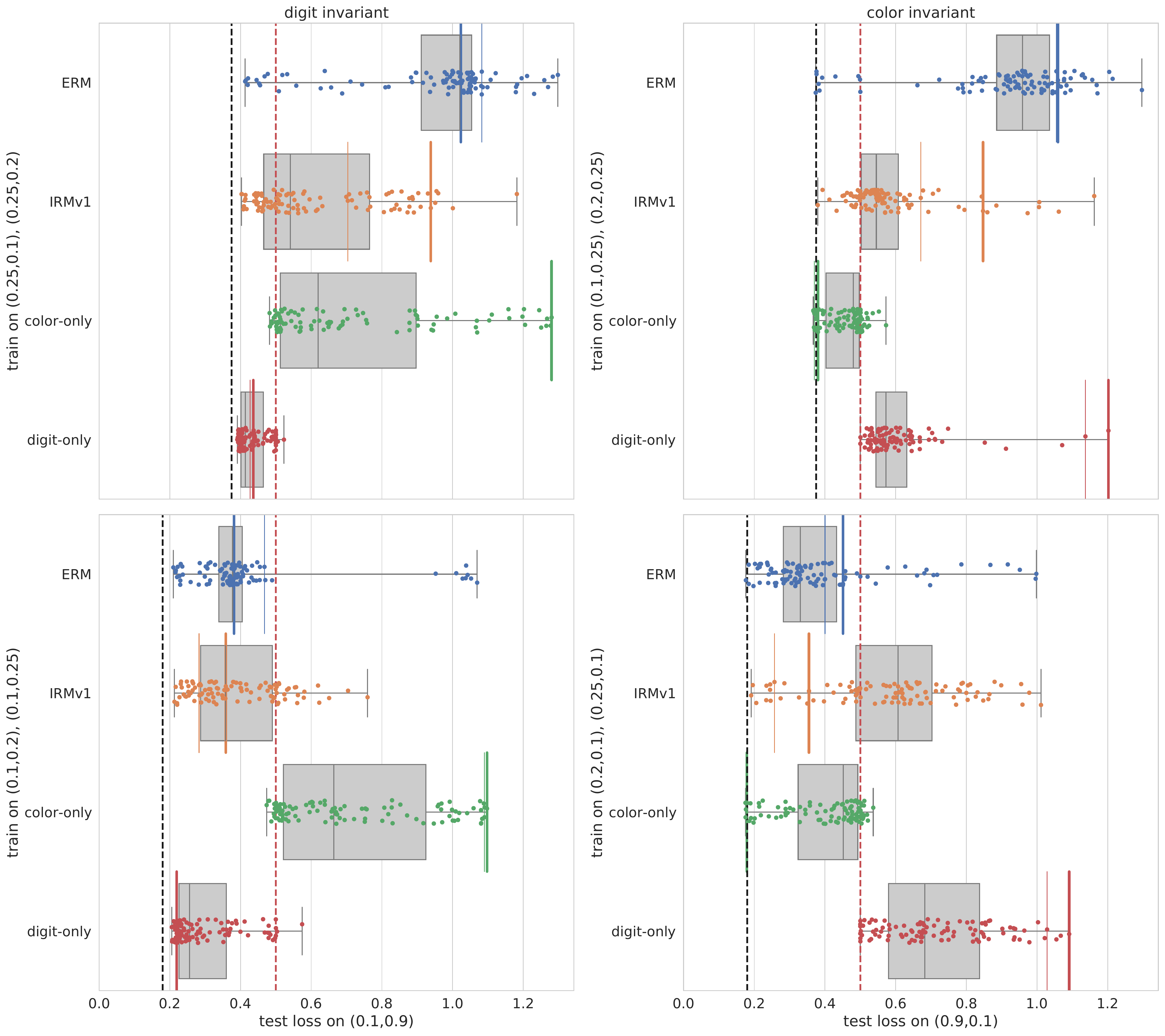} 
    \caption{\ColoredMNIST using $\sqloss$, with a ``split'' architecture.}
    \label{fig:square:split-full}
\end{figure}

\begin{figure}
    \centering
    \includegraphics[width=\textwidth]{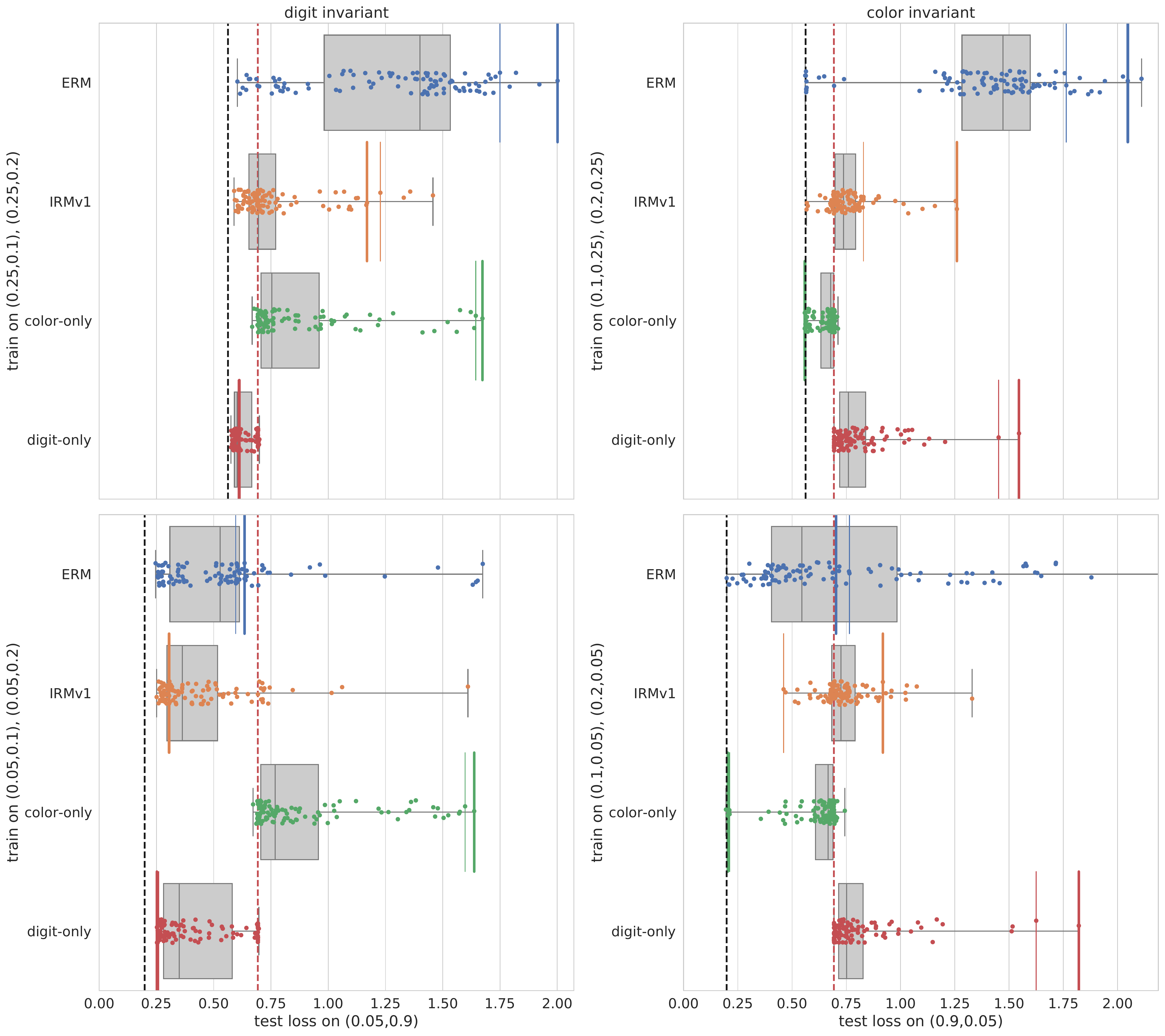}  
    \caption{\ColoredMNIST using $\logloss$, with a ``split'' architecture.}
    \label{fig:nll:split-full}
\end{figure}

\begin{figure}
    \centering
    \includegraphics[width=\textwidth]{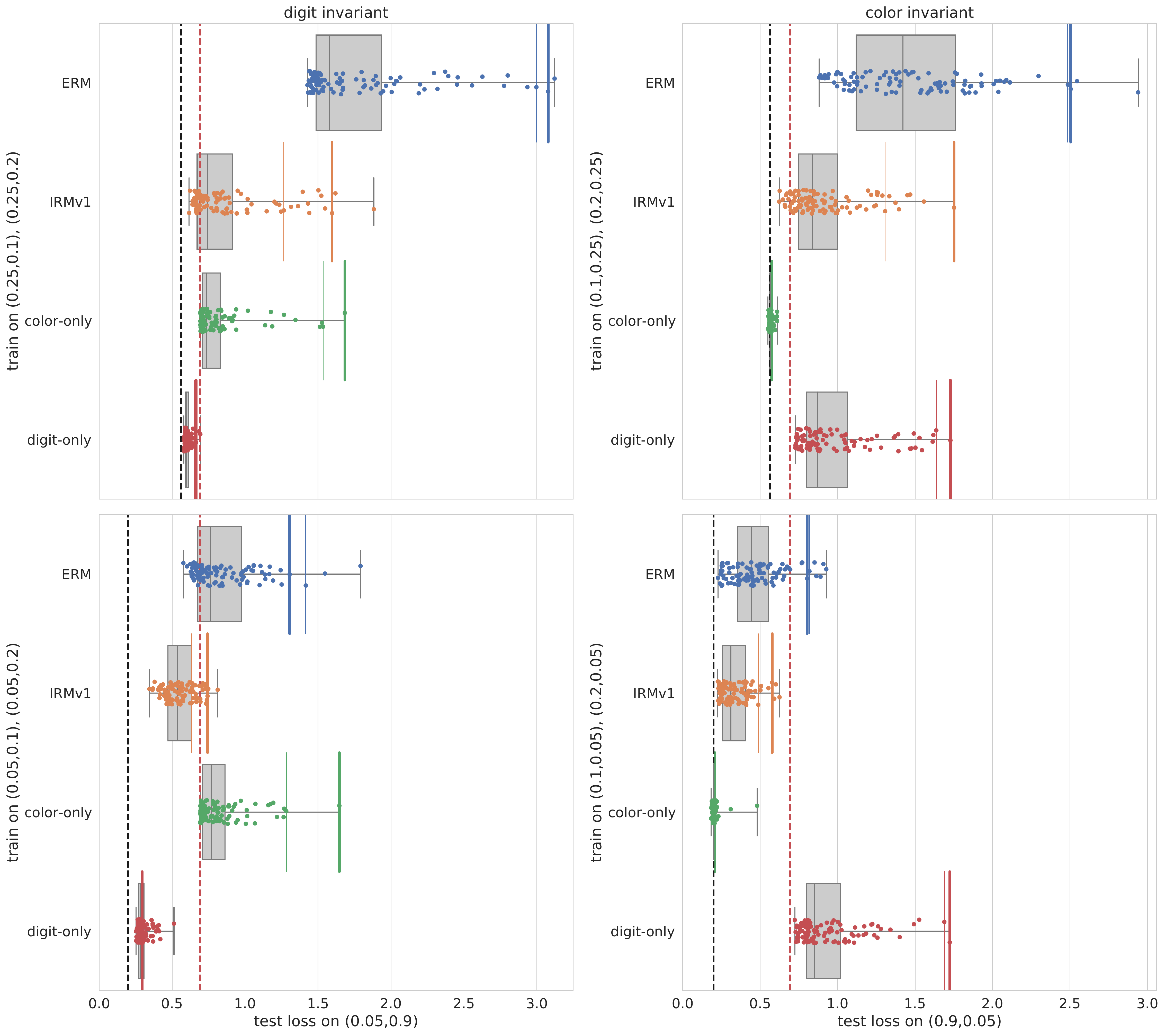} 
    \caption{\ColoredMNIST using $\logloss$, with the architecture of \citet{arjovsky19invariant}.}
    \label{fig:nll:normal-full}
\end{figure}

\clearpage
\section{More details on failure of \texorpdfstring{$\IRM$}{IRM} (\texorpdfstring{\cref{sec:invariance-sufficient}}{Section~\ref{sec:invariance-sufficient}})}
\label{appendix:inv-suff-details}

We first prove \cref{prop:inv-insuff}, restated below for convenience.

\propinvinsuff*
\begin{proof}
Since the parameterization of the environments is analytic and $\Theta = (-1/6, 1/3)$ is a connected open set, we get from part (i) of \autoref{prop:analytic-invariance} that for almost all $\Etr\subseteq \calE$ with $|\Etr| \ge 2$, it holds that $\calI(\Etr) = \calI(\calE)$. We now establish the second part: any $\pred\in \calI(\calE)$ depends on at most one of $x_1$ or $x_2$.

Similar to \cref{tab:subsets}, we can compute $\Ex_{\calD_e}[Y \mid (X_1, X_2) \in S]$ for all possible non-empty subsets $S \subseteq \set{-1,0,1}^2$ and track which of these depend or do not depend on $\theta_e$.
Since it is cumbersome to enumerate manually over all the $511$ ($=2^9-1$) possible non-empty subsets of $\set{-1,0,1}^2$, we enumerate this symbolically, using the SymPy package in Python, to identify all the subsets where $\Ex_{\calD_e}[Y \mid (X_1, X_2) \in S]$ does not depend on $\theta_e$; note that $\Ex_{\calD_e}[Y \mid (X_1, X_2) \in S]$ is a rational function in $\theta_e$ and hence if it is not identically zero, then it is in fact different for {\em almost all} pairs of choices for $\theta_e$. (Code is in the supplementary material; \texttt{two-bit/pure-irm-fail-example.py}.)

There turn out to be $37$ non-empty subsets $S$ for which $\Ex_{\calD_e}[Y \mid (X_1, X_2) \in S]$ does not depend on $\theta_e$; out of which $\Ex_{\calD_e}[Y \mid (X_1, X_2) \in S]$ is non-zero for only $6$ choices of $S$ as given in \cref{tab:subsets-inv-insuff}.

\begin{table}[h]
\begin{center}{\renewcommand{\arraystretch}{1.5}
\scriptsize
\begin{tabular}{|ccccccccc|c|c|}
\hline
\multicolumn{9}{|c|}{Subset $S \subseteq \set{-1,0,+1}^2$} & $\Ex_{\calD_{e}}[Y | X \in S]$ & Characterization of $S$ \\
\hline
&  & (+1,-1) & & & (+1,0) & & & (+1,+1) & $\phantom{+}0.3$ & {$X_1=+1$}\\
(-1,-1) & & & (-1,0) & & & (-1,+1) & & & $-0.3$ & {$X_1=-1$}\\
& & & & & & (-1,+1) & (0,+1) & (+1,+1) & $\phantom{+}0.3$ & {$X_2=+1$}\\
(-1,-1) & (0,-1) & (+1,-1) & & & & & & & $-0.3$ & {$X_2=-1$}\\
(-1,-1) & (0,-1) & & (-1,0) & (0,0) & & (-1,+1) & (0,+1) & & $-0.15$ & {$X_1\in\set{-1,0}$}\\
& (0,-1) & (+1,-1) & & (0,0) & (+1,0) & & (0,+1) & (+1,+1) & $\phantom{+}0.15$ & {$X_1\in\set{0,+1}$}\\
\hline
\end{tabular}}
\end{center}
\caption{Conditional expectations for different choices of $\varphi$ in the proof of of \cref{prop:inv-insuff}.}
\label{tab:subsets-inv-insuff}
\end{table}

For any predictor $w \circ \varphi \in \calI(\Etr)$ and any $z$ satisfying $w(z) \ne 0$, it must be the case that $\varphi^{-1}(z)$ is among the ones in \cref{tab:subsets-inv-insuff}. Thus, it is easy to see that the only predictors in $\calI(\Etr)$ are those that depend only on $x_1$, or depend only on $x_2$, or neither (for the identically zero predictor $\pred_0$). Clearly, all these predictors are also in $\calI(\calE)$ and thus, we get $\calI(\Etr) = \calI(\calE)$.

Moreover, for any environment $e \in \calE$, it holds in the case of $\ell = \sqloss$ that among all the predictors that depend only on $x_1$, the one with the lowest loss $\calL_e(\cdot)$ is $\pred_1(x) = 0.3 x_1$ and similarly, among all the predictors that depend only on $x_2$, the one with the lowest loss $\calL_e$ is $\pred_2(x) = 0.3 x_2$. (Similar, argument holds for $\ell = \logloss$.)
Thus, $\IRM$ will always pick one among $\pred_1$ and $\pred_2$.
\end{proof}

Finally, we visualize the loss of the predictors $\pred_1(x) := 0.3 x_1$, $\pred_2(x) = 0.3x_2$ and the zero predictor $\pred_0(x) = 0$ over all choices of $\theta_e \in (-1/6,1/3)$ in \cref{fig:inv-insuff-losses}. It is easy to see from the figure that if $\Etr$ only contains environments $e$ corresponding to $\theta_e < 0$, we will have that $\calL_e(\pred_2) < \calL_e(\pred_1)$ for all $e \in \Etr$, and yet the invariant predictor that minimizes $\sup_{e \in \calE} \calL_{e}(\cdot)$ is $\pred_1$ and in fact $\sup_{e\in\calE} \calL_e(\pred_2) = \sup_{e\in\calE} \calL_e(\pred_0) = 0.5$, that is, worst-case over all environments, $\pred_2$ is no better than the identically zero predictor.

\begin{figure}
\centering
\begin{lpic}[t(4mm),r(5mm)]{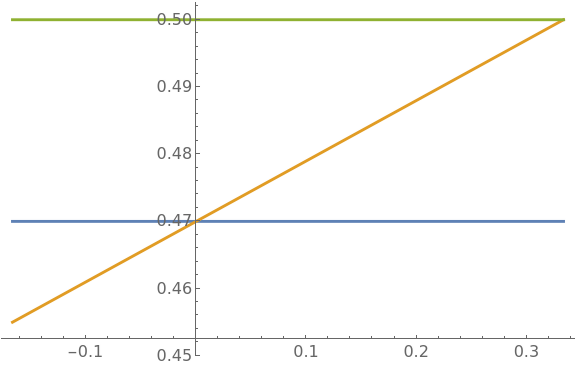(0.6)}
\normalsize
\lbl[c]{50,98;\textcolor{black!60}{\boldmath $\calL_e(\pred)$}}
\lbl[c]{153,6;\textcolor{black!60}{\boldmath $\theta_e$}}
\lbl[c]{27,43;\textcolor{WolframOne}{\boldmath$\pred_1$}}
\lbl[c]{27,17;\textcolor{WolframTwo}{\boldmath$\pred_2$}}
\lbl[c]{27,83;\textcolor{WolframThree}{\boldmath$\pred_0$}}
\end{lpic}
\caption{Losses $\calL_e$ (for $\ell = \sqloss$) of predictors $\pred_1$, $\pred_2$ and the zero predictor $\pred_0$ in $\calI(\calE)$ for $\theta_e \in (-1/6, 1/3)$.}
\label{fig:inv-insuff-losses}
\end{figure}

\section{More details on Generalization of Invariance (from \texorpdfstring{\cref{sec:invariance-generalize}}{Section \ref{sec:invariance-generalize}})} \label{appendix:analytic-invariance}

\propcontinuousinvariance*
\begin{proof}
Consider the two-bit environments of \cref{sec:bit-example},
denoted $(\alpha_e, \beta_e)$.
Define $\Pi$ as the continuous, piecewise-linear map
\[ \Pi(\theta) = \begin{cases} 
	\left( \frac{1}{10}, \frac{6\theta}{5} \right) & 0 < \theta \leq \frac14 \\
	\left( \frac{6\theta - 1}{5}, \frac{3}{10} \right) & \frac14 < \theta < 1
\end{cases}
.\]
Consider $\Thetatr = \set{\theta : 0 < \theta < \frac14}$.
Then the representation $\varphi_1(X) := X_1$ is invariant across $\Etr$,
because $\Ex_{\calD_e}[Y \mid X_1=x_1]$ is invariant across $\Etr$. Thus, in the case of $\sqloss$, the predictor $\pred_1(X) := 0.8 X_1$ is in $\calI(\Etr)$.
However, $\pred_1 \notin \calI(\calE)$,
because $\Ex_{\calD_e}[Y \mid X_1=x_1]$ changes on environments in $\calE \setminus \Etr$ when $\frac14 < \theta < 1$.
\end{proof}

We now prove \cref{prop:analytic-invariance}, restated below for convenience. First, we recall a basic fact about analytic functions.

\begin{fact}[\cite{mityagin2015zero}]\label{fact:analytic-zeros}
	Let $\Theta$ be a connected, open subset of $\bbR^d$. The set of zeros $\set{z \in \Theta \mid g(z) = 0}$
	of an analytic function $g : \Theta \to \bbR$
	has non-zero Lebesgue measure in $\bbR^d$ if and only if $g$ is identically $0$.
\end{fact}

\propanalyticinvariance*
\begin{proof}
\textbf{Part (i).} We have $\calI(\calE_{\Pi}) \subseteq \calI(\Etr)$ by definition. We establish the converse by showing that $F := \set{(\theta_1, \theta_2) \mid \calI(\set{\Pi(\theta_1), \Pi(\theta_2)}) \ne \calI(\calE_{\Pi})}$ has measure zero in $\Theta \times \Theta$.

For any $S \subseteq \calX$ define the analytic functions $n_S$ and $d_S$ as
\begin{gather*}
    n_S(\theta) := \sum_{x \in S} \sum_{y \in \calY} y \cdot \Pi_{x,y}(\theta)
    ,\qquad
    d_S(\theta) := \sum_{x \in S} \sum_{y \in \calY} \Pi_{x,y}(\theta)
    \\
    \text{so that }\,
    \Ex_{\Pi(\theta)}[Y \mid X \in S] = \frac{n_S(\theta)}{d_S(\theta)}
    \,\text{ whenever }
    \Pr_{\Pi(\theta)}[X \in S] = d_S(\theta) \ne 0
.\end{gather*}
We say that $S$ is ``valid'' if either (i) $d_S(\theta) = 0$ for all $\theta \in \Theta$, or (ii) there exists $c_S \in \bbR$ such that $n_S(\theta) = c_S \cdot d_S(\theta)$ for all $\theta \in \Theta$ subject to $d_S(\theta) \ne 0$. Note that $w \circ \varphi \in \calI(\calE_{\Pi})$ if and only if, (i) for all $z \in \mathsf{range}(\varphi)$, it holds that $\varphi^{-1}(z) \subseteq \calX$ is valid, and (ii) $w(z) = \Ex_{\Pi(\theta)}[Y | X \in \varphi^{-1}(z)]$ for any $\theta \in \Theta$ such that $d_S(\theta) \ne 0$ (in the case of $\sqloss$).

For any invalid set $S \subseteq \calX$, define
$F_S$ to consist of all pairs $(\theta_1, \theta_2)$ for which at least one of the following condition holds: (i) $d_S(\theta_1) = 0$, or (ii) $d_S(\theta_2) = 0$ or (iii) $n_S(\theta_1) \cdot d_S(\theta_2) - n_S(\theta_2) \cdot d_S(\theta_1) = 0$. Since $S$ is not valid, it follows from \autoref{fact:analytic-zeros} that $F_S$ has zero Lebesgue measure.

Finally, we show that $\displaystyle F \subseteq \bigcup_{S \subseteq \calX : S \text{ is invalid}} F_S$. For any $(\theta_1, \theta_2) \in F$ and any $w \circ \varphi \in I(\set{\Pi(\theta_1), \Pi(\theta_2)}) \smallsetminus I(E)$, there exists $z \in \mathsf{range}(\varphi)$ such that $S = \varphi^{-1}(z) \subseteq \calX$ is invalid. This implies $(\theta_1, \theta_2) \in F_S$. Since there are only finitely many $S \subseteq \calX$, we get that $F$ also has zero Lebesgue measure, thereby concluding the proof of part (i).

\textbf{Part (ii).} We have $\Iscalar(\calE_{\Pi}) \subseteq \Iscalar(\Etr)$ by definition.
To show the converse, consider any predictor $\pred = 1 \cdot \varphi \in \Iscalar(\Etr)$,
and consider
\begin{align*}
	g(\theta)
	&~:=~ \nabla_{w | w=1}\calL_{\Pi(\theta)} (w \cdot \varphi)
	~\phantom{:}=~ \sum_{x \in \calX} \sum_{y \in \calY} \Pi_{x,y}(\theta) \cdot \nabla_{w | w=1} \ell(w \cdot \varphi(x), y)%
	.\end{align*}
$g(\theta)$ is linear in  $\set{\Pi_{(x,y)}(\theta) \mid (x,y) \in \calX \times \calY}$,
each of which is analytic in $\theta$; thus $g$ is analytic in $\theta$.
Since \eqref{eq:grad-0-scalar} holds for all $e \in \Etr$, $g(\theta)=0$ for all $\theta \in \Thetatr$.
But since $\Thetatr$ has non-zero Lebesgue measure in $\bbR^d$,
by \cref{fact:analytic-zeros} $g$ is identically $0$ on $\Theta$, hence $\pred \in \Iscalar(\calE_{\Pi})$.
\end{proof}

We show how to extend \cref{prop:analytic-invariance} to the case of infinite (measurable) spaces $\calX$ and $\calY \subseteq \bbR$, where $|y| \le B$ for all $y \in \calY$ for some known bound $B$. Similar to before, let $\Delta_{\calX \times \calY}$ be the set of all probability measures over $\calX \times \calY$. For simplicity, we use $\Omega$ to denote $\calX \times \calY$.

\begin{definition}\label{def:analytic-Pi}
For $\Theta \subseteq \bbR^d$ and a measurable space $\Omega$, the parameterization $\Pi : \Theta \to \Delta_{\Omega}$ is said to be {\em analytic} if for every measurable set $S \subseteq \Omega$ and every measurable function $g : \Omega \to \bbR$, the function
\[ \Pi^g_S(\theta) \coloneqq \int_{\omega \in S} g(\omega) \ \mathsf{d}\Pi_{\theta}(\omega) \]
is an analytic function in $\theta$ (where we use $\Pi_{\theta}$ to denote the measure $\Pi(\theta)$ for simplicity).
\end{definition}

We now state the extension of \cref{prop:analytic-invariance} to the case of infinite (measurable) spaces. In the case of $\Iscalar(\calE)$, we will focus on the representations $\varphi : \calX \to \bbR$ where $|\varphi(x)| \le B$ for all $x\in \calX$. From the point of view of $\IRM_{\Wscalar}$, this is without loss of generality because we know that $|y| \le B$ for all $y\in \calY$.

\cref{prop:analytic-invariance-infinite}, however, requires a far stronger condition for $\calI(\Etr) = \calI(\calE_{\Pi})$: $\Thetatr$ needs non-zero Lebesgue measure, rather than simply almost all sets of at least two environments as in \cref{prop:analytic-invariance}. The key step in the proof of \cref{prop:analytic-invariance} that allowed for this stronger statement was that the number of subsets $S \subseteq \calX$ is finite. We do not know if \cref{prop:analytic-invariance-infinite} can be strengthened to hold for finite $\Thetatr$; if not, it will be interesting to determine other conditions under which we can get generalization of invariance for finite $\Thetatr$.

\begin{proposition}\label{prop:analytic-invariance-infinite}
\improve{Can we do this for the new, stronger \cref{prop:analytic-invariance}?}%
Let $\Thetatr \subseteq \Theta \subseteq \bbR^d$, where $\Theta$ is a connected, open set and $\Thetatr$ has non-zero Lebesgue measure, in $\bbR^d$. Suppose $\Pi : \Theta \to \Delta_{\calX \times \calY}$ is analytic (as in \cref{def:analytic-Pi}),
and $\Etr = \Pi(\Thetatr)$. Then for the $\sqloss$ loss,
\[ \text{(i) } \calI(\Etr) = \calI(\calE_{\Pi})
\quad \text{ and } \quad
\text{(ii) } \Iscalar(\Etr) = \Iscalar(\calE_{\Pi}) .\]
\end{proposition}
\begin{proof} The proof is similar to that of \cref{prop:analytic-invariance}.

\textbf{Part (i).} We have $\calI(\calE_{\Pi}) \subseteq \calI(\Etr)$ by definition. To show the converse, consider any $\pred = w \circ \varphi \in \calI(\Etr)$, with $\varphi$ invariant over $\Etr$.
For any $z$ in the range of $\varphi$, consider the function
\[
g_z(\theta)
~=~ \Ex_{\Pi(\theta)}[Y \mid \varphi(X) = z]
~=~ \frac{%
	\int_{\varphi^{-1}(z) \times \calY}\ y \, \mathsf{d}\Pi_{\theta}(x,y)%
}{%
	\int_{\varphi^{-1}(z) \times \calY} \ \mathsf{d}\Pi_{\theta}(x,y)}
.\]

Let $n_z(\theta)$ and $d_z(\theta)$ denote the numerator and denominator of $g_z(\theta)$, respectively, both of which are analytic in $\theta$ by \cref{def:analytic-Pi} (and boundedness of $y$). By \cref{obs:expectations}, there exists a constant $\alpha$ such that for all $\theta \in \Thetatr$ that satisfy $d_z(\theta) \ne 0$, it holds that
\[ 
g_z(\theta) =  \frac{n_z(\theta)}{d_z(\theta)} = \alpha \implies h_z(\theta):= n_z(\theta) - \alpha \cdot d_z(\theta) = 0 
.\]

Moreover, $d_z(\theta) = 0$ implies $n_z(\theta) = 0$, hence $h_z(\theta)=0$ for all $\theta \in \Thetatr$. Since $\Thetatr$ has non-zero Lebesgue measure, it follows from \cref{fact:analytic-zeros} that $h_z(\theta)$ that is identically zero on $\Theta$. This implies for all $\theta \in \Theta$ such that $d_z(\theta) \neq 0$, $g_z(\theta) = \alpha$. Hence, by \cref{obs:expectations}, we get that $\pred \in \calI(\calE_{\Pi})$.

\textbf{Part (ii).} This follows similarly. We have $\Iscalar(\calE_{\Pi}) \subseteq \Iscalar(\Etr)$ by definition. To show the converse, consider any predictor $\pred = 1 \cdot \varphi \in \Iscalar(\Etr)$,
and consider the following function of $\theta$:
\[g(\theta) ~\coloneqq~ \nabla_{w|w=1} \calL_{\Pi(\theta)}(w\cdot \varphi) ~=~ \int_{\calX \times \calY} \nabla_{w|w=1} \ell(w\cdot \varphi(x),y) \, \mathsf{d}\Pi_\theta(x,y)\, , \]
which by \cref{def:analytic-Pi} is an analytic function in $\theta$. To derive this, we swapped the $\nabla_{w|w=1}$ with $\int_{\calX \times \calY}$, possible because $|y|$ and $|\varphi(x)|$ are uniformly bounded \citep[Theorem 2]{planetmathdiffint}.

Since \eqref{eq:grad-0-scalar} holds for all $e \in \Etr$, $g(\theta)=0$ for all $\theta \in \Thetatr$.
But since $\Thetatr$ has non-zero Lebesgue measure in $\bbR^d$,
we have from \cref{fact:analytic-zeros} that $g$ is identically $0$ on $\Theta$, hence $\pred \in \Iscalar(\calE_{\Pi})$.
\end{proof}

\end{document}